\documentclass[accepted]{uai2026}

\usepackage[american]{babel}

\usepackage{natbib}
\usepackage{mathtools}
\usepackage{booktabs}
\usepackage{tikz}
\usepackage{tikz-cd}
\usepackage[dvipsnames]{xcolor}
\usepackage{tikz-3dplot}
\usetikzlibrary{calc,arrows.meta}
\usepackage{shortcuts}
\usepackage{booktabs}
\usepackage{amsfonts}
\usepackage{nicefrac}
\usepackage{microtype}
\usepackage{multirow}
\usepackage{booktabs}
\usepackage{wrapfig}
\usepackage{algorithm}
\usepackage{algpseudocode}
\usepackage{makecell}
\usepackage{comment}
\usepackage{graphicx}
\usepackage{amsthm}
\usepackage{caption}

\definecolor{mydarkblue}{rgb}{0,0.08,0.45}
\definecolor{theoblue}{rgb}{0.19,0.15,0.65}
\hypersetup{colorlinks=true,
    linkcolor=theoblue,
    citecolor=theoblue,
    filecolor=theoblue,
    urlcolor=theoblue}

\newcolumntype{L}[1]{>{\raggedright\let\newline\\\arraybackslash\hspace{0pt}}m{#1}}
\newcolumntype{C}[1]{>{\centering\let\newline\\\arraybackslash\hspace{0pt}}m{#1}}
\newcolumntype{R}[1]{>{\raggedleft\let\newline\\\arraybackslash\hspace{0pt}}m{#1}}

\newtheorem{innercustomthm}{Theorem}
\newenvironment{customthm}[1]
  {\renewcommand\theinnercustomthm{#1}\innercustomthm}
  {\endinnercustomthm}

\newtheorem{innercustomlem}{Lemma}
\newenvironment{customlem}[1]
  {\renewcommand\theinnercustomlem{#1}\innercustomlem}
  {\endinnercustomlem}

\newtheorem{innercustomdef}{Definition}
\newenvironment{customdef}[1]
  {\renewcommand\theinnercustomdef{#1}\innercustomdef}
  {\endinnercustomdef}

\theoremstyle{plain}
\newtheorem{example}{Example}

\DeclareRobustCommand{\parhead}[1]{\textbf{#1}~}

\title{Unsupervised Causal Abstractions Discovery}

\author[1,2]{\href{mailto:<theo.saulus@mila.quebec>?Subject=About Unsupervised Causal Abstractions Discovery}{Théo~Saulus}{}}
\author[1,2,3]{Simon~Lacoste-Julien}
\author[1,2,3]{Dhanya~Sridhar}
\affil[1]{
    Mila - Quebec AI Institute
}
\affil[2]{
    Université de Montréal
}
\affil[3]{
    Canada CIFAR AI Chair
}

\begin{document}
\maketitle

\begin{abstract}
Causal abstractions formalize when a high-level structural causal model (SCM) captures the interventional behavior of a lower-level SCM.
Existing applications of this notion largely follow a hypothesis-testing paradigm: an expert proposes a candidate high-level model and then evaluates if the low-level system implements it. We study the complementary problem of learning a high-level model directly from low-level measurements. Our contributions leverage hypotheses from low-rank causal discovery, and can be summarized as follows: (1) we show that observations generated by a low-rank graph induce latents that form a causal abstraction, (2) we provide identifiability results about these latents, and (3) we propose a practical objective to learn this high-level SCM.
\end{abstract}

\section{Introduction}\label{sec:intro}

Complex systems can be described at different levels of abstraction: from low-level measurements about the system, one can form high-level summaries of its behavior. For example, we can record the neural activity of an individual performing a given task, and group the neurons by the stimuli they respond to, yielding a coarse description in terms of functional groups rather than individual neurons.

However, many scientific queries are about predicting the effects of interventions, such as "what happens if we activate this part of the brain?". Addressing such questions requires \emph{causal} abstractions \citep{beckersAbstractingCausalModels2019,rubenstein2017causal}:
put informally for now, a causal abstraction encodes at the high-level the effects of some interventions of interest on low-level variables, satisfying a form of consistency.

In mechanistic interpretability, this viewpoint motivates hypothesis testing methods where one proposes a candidate high-level model and checks whether it is implemented by a deep neural network (DNN) \citep{geigerCausalAbstractionsNeural2021, geigerFindingAlignmentsInterpretable2024, wu2023interpretability}. The underlying hope is that once we know the high-level computations the DNN implements, we can translate interventions on high-level variables into corresponding interventions on neurons to steer the model behavior.
However, enumerating over candidate abstractions to test is costly, and worse, not possible when we study systems that we do not fully understand (for example, a DNN trained to solve protein folding).
Motivated by these applications in interpretability and steering, this paper is about the unsupervised discovery of such causal abstractions.

Unsupervised objectives are typically under-specified and admit many equivalent solutions, thus requiring further constraints or assumptions.
In this work, we discover causal abstractions for a family of low-level systems whose graphs exhibit a low-rank structure.
Specifically, we consider factor directed acyclic graphs (factor DAGs, or f-DAGs), a family of models over low-level variables that are parameterized by a small number of latent factors, as introduced by \citet{lopezLargeScaleDifferentiableCausal2022}.
This low-rank assumption is well-suited to modeling cell or neural biology where many low-level variables like genes or neurons carry out similar functions. Likewise in DNNs, because they perform computations in a distributed way, neurons are also likely to cluster into groups that capture the same abstract ``concept'' in a higher-level computation graph.

The goal of \citet{lopezLargeScaleDifferentiableCausal2022} was to learn a causal graph over low-level variables.
They show that when observations come from a model that satisfies the f-DAG property, the DAG that best explains these low-level observations is likely to be unique and can be learned from data.
Our goal with this paper is not to propose a new algorithm, but to prove that this causal discovery algorithm identifies a causal model over latent factors that is a causal abstraction of the low-level system. Additionally, we show that some assumptions allow us to make these latent factors identifiable up to trivial equivalence.
We call this the \emph{anchor assumption}.
Loosely speaking, it requires that for each latent factor, there exist low-level variables that uniquely influence and are influenced by this factor.
These anchor variables do not need to be known, but by ensuring that the solutions to Boolean matrix factorization satisfy the anchor assumption, we guarantee that the overall unsupervised objective only admits solutions that are simple reparameterizations of the true model over latent factors.
We validate the theory empirically by studying simulated settings where the target causal abstraction is known, and also study a fully unsupervised case where we qualitatively analyze abstractions learned from a DNN that solves an arithmetic task.

\section{Related work}

Causal abstractions formalize when a high-level SCM preserves the interventional behavior of a lower-level SCM through a state map and an intervention map \citep{rubenstein2017causal, beckersAbstractingCausalModels2019, beckers2019approximatecausalabstraction}. This formalism has been used in mechanistic interpretability to test whether a proposed high-level causal model is implemented by a neural network \citep{geigerCausalAbstractionsNeural2021, geigerFindingAlignmentsInterpretable2024, geigerCausalAbstractionTheoretical2025, wu2023interpretability}. Our work instead studies the complementary problem of discovering such a high-level model directly from low-level measurements. This follows the older motivation of causal feature learning, which seeks macro-causes and macro-effects from high-dimensional observations \citep{chalupkaMultiLevelCauseEffectSystems2015, chalupkaCausalFeatureLearning2017}, but uses the modern causal-abstraction formalism based on interventional consistency.

Several recent papers study learning causal abstractions rather than only testing expert-proposed ones. \citet{zennaro2023jointly, felekis2024causal, massidda2024learning, d2025causal} learn abstraction maps between low- and high-level causal models, but they all require data or distributional information from \emph{both} levels. In contrast, we discover the high-level latent SCM from low-level observations only. \citet{zhu2024unsupervised} are close to our work in their goal: they also consider unsupervised causal abstraction when the high-level variables are unknown. Their setting, however, relies on linear SCMs, shift interventions, linear state and intervention maps, and constructive abstractions, which implies that each low-level variable can only belong to one latent factor. In contrast, we identify high-level latent factors from low-level measurements through a low-rank hypothesis and weaker identifiability assumptions (the anchor condition). Targeted reduction of causal models \citep{kekictargeted} learns a reduced high-level causal model and a map from low-level shift interventions to high-level interventions from interventional data. However, it is target-centered: the high-level model is built around a known variable of interest, rather than discovering a general latent-factor SCM.

Our structural assumptions connect causal abstraction discovery to low-rank causal discovery and latent-variable learning. We build on factor DAGs \citet{lopezLargeScaleDifferentiableCausal2022}, where a low-rank adjacency structure is represented by latent factors mediating low-level causal relations. Low-rank causal discovery has also been studied as a structural restriction on causal graphs \citep{dong2022graphs, fang2023low}. Structural Causal Bottleneck Models \citep{bing2026structural}, are closely related in that they posit that causal effects between high-dimensional variables can factor through low-dimensional bottlenecks, and they provide identifiability results for such bottlenecks. However, their framework treats bottlenecks primarily as task-specific sufficient summaries for estimating causal effects in a known graph, rather than as a learned high-level SCM equipped with explicit state and intervention maps satisfying a causal-abstraction criterion.

Our anchor assumption, inspired by topic modeling \citep{arora2012learningtopicmodels}, can be related to pure-child assumptions in causal representation learning. For example, \citet{xie2020generalized} identify latent causal graphs in linear non-Gaussian measurement models using pure observed children of latent variables; \citet{markham2023neuro} introduce neuro-causal factor analysis, which learns latent common causes of observed measurements. These resemble our out-anchor condition, but our setting differs in that our observed variables also mediate latent factors relations. In different setups, other works about the identifiability of deep generative models \citep{moran2022identifiable, lee2026deep} or about the identifiability of Boolean matrix factorization \citep{mironBooleanDecompositionBinary2021, desantisFactorizationBinaryMatrices2021} also resort to very similar assumptions.

\section{Background notions}
In this section, we define structural causal models \citep{pearl2009causality} and introduce their properties before formalizing causal abstractions.

\begin{definition}[Structural causal model]
A structural causal model (SCM) is a tuple $\mM = (\mV, \mU, \mF)$, where:
\begin{itemize}
    \item $\mV$ is a set of \emph{endogenous} variables taking values in \(\Val_\mV := \bigcup_{V\in \mV}\Val_V\). We denote by $\mbV \subset\mV$ a subset of these variables. Realizations of $V$ (respectively $\mbV$) are denoted $v$ (resp. $\mbv$).
    \item $\mU$ is a set of \emph{exogenous} variables (or \emph{context}), taking values in \(\Val_\mU := \bigcup_{U\in \mU}\Val_U\). We denote by $\mbU \subset\mU$ a subset of these variables. Realizations of $U$ (respectively $\mbU$) are denoted $u$ (resp. $\mbu$).
    \item for each $V\in \mV$, we define a \emph{mechanism} \(F_V \in \mF : \Val_{\Pa_V}\times \Val_\mU \to \Val_V,\) where $\Pa_V\subseteq \mV\setminus\{V\}$ is the set of \emph{endogenous parents} of $V$.
\end{itemize}
\end{definition}
Crucially, SCMs provide a formalism for interventions.

\begin{definition}[Intervention]\label{def:intervention}
    An \emph{intervention} $i$ is described by a support $\mbI\subseteq\mV$ and a family of functions $\{I_V\}_{V\in \mbI}$ such that for each $V\in\mbI$, $I_V : \Val_{\Pa_V}\times \Val_\mU \to \Val_V$ replaces the original mechanism $F_V$.
    An intervention is \emph{hard} if the functions $I_V$ are constant assignments, which we denote $i\in\Hard_\mbI$. Otherwise, it is \emph{soft}, which we denote $i\in\Soft_\mbI$.
\end{definition}
Following \citet{geigerCausalAbstractionTheoretical2025, eberhardt2006interventionsandcausalinference}, we generalize the notion of soft interventions:
\begin{definition}[Parametric interventions]\label{def:interventional}
    A \emph{parametric intervention} $i$ is described by a support $\mbI\subseteq\mV$ and a family of functions $\{I_V\}_{V\in \mbI}$ such that for each $V\in\mbI$, $I_V : \Soft_V \to \Soft_V$ replaces the original mechanism $F_V$ with $I_V\langle F_V\rangle$. We denote $i\in\Para_\mbI$.
\end{definition}
In words, parametric interventions\footnote{Called \emph{interventionals} by \citet{geigerCausalAbstractionTheoretical2025}.} depend on the original mechanisms of the SCM to perform their replacement.
They extend the notion of soft interventions, which are simply ``constant'' parametric interventions.
In the following, we denote $\mI$ the set of admissible interventions on $\mM$, and we denote $\mM^i$ the SCM subject to the intervention $i\in\mI$.

Given a context assignment $\mbu\in\Val_\mU$, the SCM defines a system of equations
\(\forall V\in \mV, \ V = F_V\bigl(\mbv_{\Pa_V}, \mbu\bigr)\)
whose solution (when it exists) is an endogenous assignment $v\in\Val_V$.
We can use the graphical properties of an SCM to find its solutions \citep{halpern2000axiomatizing}.
\begin{definition}[Causal graph]
    The \emph{causal graph} of $\mM$ is the directed graph $G=(\mV,\mE)$ defined such that \[(V_1\to V_2)\in \mE \quad\iff\quad V_1\in \Pa_{V_2}\]
    The graph $G$ can be encoded by its adjacency matrix $A\in\{0,1\}^{n\times n}$ with $A_{ij}=1 \iff (V_i\to V_j)\in\mE$.
\end{definition}

\begin{definition}[$\Solve$ operator]
For an intervention $i\in\mI$ and a context $\mbu\in \Val_\mU$ we denote by \(\Solve(\mM^{i};\mbu)\in \Val_\mV\) a solution to the system of equations defined by the SCM $\mM^i$.
\end{definition}
\begin{assumption}[Strong acyclicity]\label{assu:strong_acyc}
    The causal graph of $\mM$ is independent of context $\mU$\footnote{This notion is sometimes known as \emph{strong recursivity} in the literature, e.g. in \citet{beckersAbstractingCausalModels2019}.}.
\end{assumption}
In the setting we study, we will always assume strong acyclicity. Thus, we evaluate the mechanisms in topological order to find the unique solution to the $\Solve$ operator.

Finally, we can extend SCMs to the probabilistic case.
\begin{definition}[Probabilistic causal models]\label{def:prob_causal_model}
    A probabilistic causal model is a tuple $(\mM, \Pr)$ where $\mM$ is a SCM and $\Pr$ is a probability distribution on contexts $\mU$.
\end{definition}
In this setting, we consider the induced distribution on $\Solve(\mM^{i};\mbu)$, following \citet{beckersAbstractingCausalModels2019}:
\begin{align*}
    \Pr_{\mM^i}(\mbv):= \Pr(\{\mbu : \Solve(\mM^{i};\mbu) = \mbv\})
\end{align*}

\parhead{Causal abstractions.} A causal abstraction relates a low-level SCM $\mL=(\mX,\mU_L,\mF_L)$ and a high-level SCM $\mH = (\mZ,\mU_H,\mF_H)$ (typically containing much fewer variables), such that they both describe the same causal phenomena \citep{rubenstein2017causal}.
\begin{definition}[Exact transformations]\label{def:proba_exact_transfo}
    Let $(\mL,\Pr_L)$ and $(\mH,\Pr_H)$ be two probabilistic causal models.
    Let $\tau : \Val_\mX \to \Val_\mZ$ and $\omega : \mI_L \to \mI_H$ be two partial surjective functions where $\omega$ is order preserving.

    $(\mH,\Pr_H)$ is an \emph{exact $(\tau - \omega)$-transformation} of $(\mL,\Pr_L)$ if
    \[\forall i\in \mI_L, \quad \Pr_{\mH^{\omega(i)}} = \tau(\Pr_{\mL^i}),\]
    where $\tau(\Pr_{\mL}) = \Pr_\mL(\{\mbx : \tau(\mbx)=\mbz\})$ denotes the pushforward distribution of $\Pr_\mL$ by $\tau$.
\end{definition}
Intuitively, for a causal abstraction to faithfully capture low-level effects at the high-level, applying $i$ at the low-level and translating outcomes via $\tau$ matches applying $\omega(i)$ directly at the high-level, which can be represented by the following commutative diagram:
\[
\begin{tikzcd}[column sep=large, row sep=small]
i \arrow[r, "\omega"] \arrow[d] & \omega(i) \arrow[d] \\
\Pr_{\mL^i} \arrow[r, "\tau"] & \Pr_{\mH^{\omega(i)}}
\end{tikzcd}
\]
We refer to $\tau$ as the \emph{state map}, and $\omega$ as the \emph{intervention map}.

\citet{beckersAbstractingCausalModels2019} observe that the notion of exact transformation can be too weak because an unfortunate choice of $\Pr_L,\Pr_H$ may hide mechanistic mismatches. They therefore introduce uniform transformations:

\begin{definition}[Uniform transformations]\label{def:uniform_transfo}
    $\mH$ is a \emph{uniform $(\tau - \omega)$-transformation} of $\mL$ if for \emph{all} distributions $\Pr_L$ on the low-level context $\mU_L$, there exists a distribution $\Pr_H$ such that $(\mH,\Pr_H)$ is an exact $(\tau - \omega)$-transformation of $(\mL,\Pr_L)$.
\end{definition}

\section{Towards causal abstractions discovery}
Most current applications of causal abstractions are related to neural network interpretability: an expert proposes a candidate high-level model $\mH$ about what logic the model might implement to solve a given task, and then tests whether the computational graph $\mL$ of a neural network implements it \citep{geigerCausalAbstractionsNeural2021, geigerFindingAlignmentsInterpretable2024, wu2023interpretability}.

Intuitively, we view $\mH$ as a model of latents that summarize the computations done by the low-level computational graph $\mL$: the value of each latent $Z\in \mZ$ can be computed with a subset of variables $\mbX\subseteq \mX$, and in turn there is another subset of variables $\mbX'\subseteq\mX$ which can essentially be determined by the value of this $Z$. In other words, $Z$ summarizes \textit{all there is to know} about the causal relations $\mbX \rightarrow \mbX'$.

Proposing a candidate high-level model $\mH$ is difficult, because it requires intuition or expertise about the task at hand: for instance, we can propose a logic for two-digit addition because we master this task.
In the following, we develop assumptions under which we can \emph{learn} a high-level model $\mH$ directly from low-level measurements by minimizing a differentiable objective.

In particular, we recall from \citet{lopezLargeScaleDifferentiableCausal2022} that for an SCM $\mM$ whose graph $G$ is low-rank, one can naturally construct a latent SCM $\mM^*$.
We prove that this latent SCM $\mM^*$ is a uniform transformation of $\mM$.
Since causal discovery tools allow us to learn the graph and mechanisms of $\mM$ from data, we propose later a practical algorithm for unsupervised causal abstraction discovery inspired by it.

\parhead{Factor directed acyclic graphs.} Consider the causal graph $G$ between observed variables $\mX$, encoded by an adjacency matrix $A\in\{0,1\}^{n\times n}$, where $n$ is the number of variables.
In causal discovery, where we seek to learn the causal graph $G$ from observations, a common structural assumption to restrict the space of candidate graphs is the low-rankness of the adjacency matrix $A$ \citep{dong2022graphs, fang2023low}.
More specifically, assuming that $A$ has Boolean rank $m<n$ means that there exists $Q,R \in \{0,1\}^{n\times m}$ such that
\[
A \ = \ Q \diamond R^\top
\ := \  \bigvee_{k=1}^{m} \bigl(Q_{*k} \wedge R_{*k}\bigr),
\]
where $\diamond$ denotes the Boolean matrix product, illustrated in Figure \ref{fig:QR_product_schema}.
Intuitively, this low-rank decomposition captures the idea that many low-level variables, like genes in a regulatory network \citep{alizadeh2000distinct} or neurons in the brain \citep{yeo2011organization}, share the same parents and children due to higher-order functions. Matrix factorization methods have previously been used to discover such higher-level functions \citep{newman2015robust}.

This factorization admits a graphical interpretation via \emph{factor directed graphs}.
\begin{definition}[Factor SCM]
    An \emph{f-SCM} is an SCM $\mM=(\mV,\mU,\mF)$ such that $\mV$ is partitioned into two subsets: observed \emph{variables} $\mX = \{X_1,\ldots,X_n\}$ and unobserved latent \emph{factors} $\mZ=\{Z_1,\ldots,Z_m\}$. Additionally, we suppose that the causal graph $G$ of an f-SCM is bipartite, directed, and strongly acyclic, and we denote it an \emph{f-DAG}.
\end{definition}

In particular, the edges in $G$ may only link variables to factors or factors to variables:
\[
G = (\mV = \mX \cup \mZ,\mE_\mV),
\quad
\mE_\mV \subset (\mX\times \mZ) \cup (\mZ\times \mX).
\]
Relating back to Boolean matrix factorization, the matrix $Q$ captures the adjacency matrix of directed edges from low-level variables to factors, i.e., $\mbX \rightarrow\mbZ$, while $R$ captures the directed edges from factors $\mbZ$ back to low-level variables $\mbX$. Thus, $Q$ and $R$ fully characterize the f-DAG.

Further, the graph $G$ can be collapsed into two distinct (``mono-partite'') graphs: one over variables $G_X$ with $$X_i\to X_j \iff \exists Z\in \mZ, \ X_i\to Z\to X_j \text{ in $G$,}$$ and one over factors $G_Z$ defined analogously.

From now on, we will distinguish 3 different notions of parents in an f-DAG: for any $Z\in\mZ,\ X\in\mX$,
$$\begin{aligned}
    X'\in\Pa_Z &\iff X'\rightarrow Z \text{ in } G \\
    Z'\in\Pa_X &\iff Z'\rightarrow X \text{ in } G \\
    X'\in\Pa^L_X &\iff X'\rightarrow X \text{ in } G_L \\
    Z'\in\Pa^H_Z &\iff Z'\rightarrow Z \text{ in } G_H \\
\end{aligned}$$
\parhead{Key idea.} The notion of f-DAGs seems to meet our intuitive desiderata for causal abstractions discovery in that they summarize the way low-level variables interact with fewer factors, and they are learnable from data with low-rank causal discovery.

The following sections formalize these two points: we show how the factor-level description can indeed be cast as a uniform transformation of the variable-level description, and we identify conditions under which the factor representation is identifiable.

\section{Theoretical results}
This section formalizes the previous insights with three results that justify the use of low-rank causal discovery to discover causal abstractions: (i) first, we show that if we knew the underlying f-SCM that generates low-level observations, we could derive a causal abstraction of the low-level model; (ii) since we do not know the underlying f-SCM, we establish assumptions under which Boolean matrix factorization can recover the latent factor variables up to simple reparameterizations; (iii) completing the story, we derive assumptions under which the full high-level SCM can be recovered up to simple reparameterizations.
We state our main results, and sketch some of the proofs; the full derivations are available in Appendix \ref{sec:proofs_theory}.

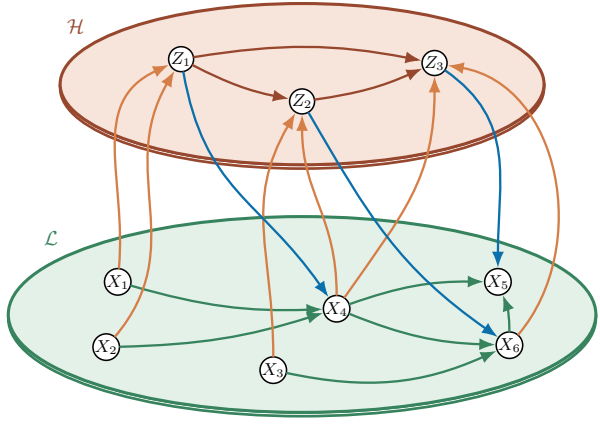
\begin{figure}
\centering
\resizebox{\columnwidth}{!}{
\newcommand{\tdcylcol}[6]{
    \path (1,0,0);
    \pgfgetlastxy{\cylxx}{\cylxy}
    \path (0,1,0);
    \pgfgetlastxy{\cylyx}{\cylyy}
    \path (0,0,1);
    \pgfgetlastxy{\cylzx}{\cylzy}
    \pgfmathsetmacro{\cylt}{(\cylzy * \cylyx - \cylzx * \cylyy)/ (\cylzy * \cylxx - \cylzx * \cylxy)}
    \pgfmathsetmacro{\ang}{atan(\cylt)}
    \pgfmathsetmacro{\ct}{1/sqrt(1 + (\cylt)^2)}
    \pgfmathsetmacro{\st}{\cylt * \ct}

    \filldraw[fill=white, draw=#5, line width=1.4pt]
      (#3*\ct+#1,#3*\st+#2,0) -- ++(0,0,-#4)
      arc[start angle=\ang,delta angle=-180,radius=#3] -- ++(0,0,#4)
      arc[start angle=\ang+180,delta angle=180,radius=#3];

    \filldraw[fill=#6, draw=#5, line width=1.6pt] (#1,#2,0) circle[radius=#3];
}

\tdplotsetmaincoords{30}{30}

\begin{tikzpicture}[
  tdplot_main_coords,
  transform shape,
  x=60pt, y=20pt,
  line cap=round, line join=round,
  >=Latex,
  var/.style={draw=black, fill=white, thick, circle, minimum size=13pt, inner sep=0pt, font=\small},
  fac/.style={draw=black, fill=white, thick, circle, minimum size=13pt, inner sep=0pt, font=\small},
  mechL/.style={->, very thick, draw=Ldraw},
  mechH/.style={->, very thick, draw=Hdraw},
  up/.style   ={->, very thick, draw=updraw},
  down/.style ={->, very thick, draw=downdraw},
]

\coordinate (Lc) at (0,0,0);
\coordinate (Hc) at (0,6,0);

\tdcylcol{0}{0}{2.55}{0.14}{Ldraw}{Lfill}      
\tdcylcol{0}{6}{2.10}{0.14}{Hdraw}{Hfill}      

\node[anchor=west, scale=1.05, text=Ldraw] at (-2.3, 2, 0) {$\mL$};
\node[anchor=west, scale=1.05, text=Hdraw] at (-2.1, 7.5, 0) {$\mH$};

\begin{scope}[shift={(Lc)}]
\node[var] (X1) at (-1.60, 0.85,0) {$X_{1}$};
\node[var] (X2) at (-1.70,-0.85,0) {$X_{2}$};
\node[var] (X3) at (-0.25,-1.45,0) {$X_{3}$};
\node[var] (X4) at ( 0.30, 0.15,0) {$X_{4}$};
\node[var] (X5) at ( 1.70, 0.85,0) {$X_{5}$};
\node[var] (X6) at ( 1.80,-0.80,0) {$X_{6}$};
\end{scope}
\begin{scope}[shift={(Hc)}]
\node[fac] (Z1) at (-1.05, 0.65,0) {$Z_{1}$};
\node[fac] (Z2) at ( 0.00, -0.45,0) {$Z_{2}$};
\node[fac] (Z3) at ( 1.15, 0.55,0) {$Z_{3}$};
\end{scope}
\draw[mechL] (X1) to[bend right=9] (X4);
\draw[mechL] (X2) to[bend right=8] (X4);

\draw[mechL] (X3) to[bend right=18] (X6);
\draw[mechL] (X4) to[bend left=10] (X5);
\draw[mechL] (X4) to[bend right=10] (X6);
\draw[mechL] (X6) to[bend right=10] (X5);

\draw[mechH] (Z1) to[bend left=10] (Z3);
\draw[mechH] (Z1) to[bend right=8] (Z2);
\draw[mechH] (Z2) to[bend right=8] (Z3);

\draw[up] (X1) to[out=85,in=-155,looseness=1.10] (Z1);
\draw[up] (X2) to[out=50,in=-120,looseness=1.10] (Z1);
\draw[up] (X3) to[out=90,in=-120,looseness=1.10] (Z2);
\draw[up] (X4) to[out=55,in=-90,looseness=1.10] (Z3);
\draw[up] (X4) to[out=90,in=-90,looseness=1.10] (Z2);
\draw[up] (X6) to[out=50,in=-10,looseness=1.05] (Z3);

\draw[down] (Z1) to[out=-80,in=125,looseness=1.15] (X4);
\draw[down] (Z2) to[out=-60,in=145,looseness=1.05] (X6);
\draw[down] (Z3) to[out=-35,in=90,looseness=1.10] (X5);

\end{tikzpicture}
}
\caption{Illustration of $\mH$ (in red), $\mL$ (in green), as well as the mechanisms $F_X$ (in blue) and $F_Z$ (in orange).}
\label{fig:HL_disks_schema}
\end{figure}

\subsection{f-SCMs build causal abstractions}
Let $\mM=(\mX\cup\mZ,\mU,\mF)$ be an f-SCM. Denote $G$ the associated f-DAG, and denote $G_L$ (resp. $G_H$) its collapse on the observed variables $\mX$ (resp. latent factors $\mZ$).
Furthermore, suppose that there is no direct link from exogenous variables $\mU$ to latent factors $\mZ$: intuitively, we think of factors as summarizing low-level measurements, and we may not want to consider that this summary stochastic.
In other words, $\mM$ admits the following mechanisms:
\begin{align*}
    &\forall X\in \mX, \ F_X : \Val_{\Pa_X} \times \Val_\mU \to \Val_X\\
    &\forall Z\in \mZ, \ F_Z : \Val_{\Pa_Z} \to \Val_Z
\end{align*}
where $\Pa_X \subseteq \mZ$ and $\Pa_Z \subseteq \mX$.

These mechanisms imply the following low-level mechanisms over $\mL=(\mX,\mU,\mF_L)$ and high-level mechanisms over $\mH=(\mZ,\mU,\mF_H)$ when collapsing on either type of node (see Figure \ref{fig:HL_disks_schema} for an illustration):
\begin{align*}
&\begin{aligned}
    F^{L}_X : \Val_{\Pa^L_X} \times \Val_\mU &\to \Val_X \\
    \mbx, \mbu &\mapsto F_X\Big(\{F_Z(\mbx_{\Pa_Z})\}_{Z\in \Pa_X},\mbu\Big) \\
\end{aligned}
\\
&\begin{aligned}
    F^{H}_Z : \Val_{\Pa^H_Z} \times \Val_\mU &\to \Val_Z \\
    \mbz, \mbu &\mapsto F_Z\Big(\{F_X(\mbz_{\Pa_X}, \mbu)\}_{X\in \Pa_Z}\Big) \\
\end{aligned}
\end{align*}
where $\Pa_X^L = \bigcup_{Z\in \Pa_X} \Pa_Z \subseteq \mX$ and $\Pa_Z^H = \bigcup_{X\in \Pa_Z} \Pa_X \subseteq \mZ$.
That is, we compose together how low-level variables define factors and how those factors define downstream low-level variables.
Note that $\mL$ and $\mH$ share here the same context $\mU$, coming from the f-SCM $\mM$.

\begin{example}
    As a concrete example, based on Figure \ref{fig:HL_disks_schema}, the mechanism $F^H_{Z_3}$ for the high-level variable $Z_3$ in the collapsed model $\mH$ is a function of $Z_1$ to $Z_2$ such that $F^H_{Z_3} = F_{Z_3}(F_{X_4}(Z_1), \ F_{X_6}(Z_2))$.
\end{example}

\parhead{Key idea.} The first result is to show that if we knew the f-SCM that generates observations of low-level variables, the corresponding collapsed SCM $\mH$ over the factors alone is a valid causal abstraction of the collapsed SCM $\mL$ over low-level variables alone. To derive this result, we will show that $\mH$ is an exact transformation of $\mL$. To show this property, we will first need to define all the relevant mappings from low-level to high-level.

\parhead{Intervention sets.}
We consider $\mI_L$ to be a subset of all possible hard interventions $\Hard_\mX$,
and $\mI_H$ a subset of parametric interventions $\Para_\mZ$ that we characterize below.

\parhead{State map.}
The state map $\tau : \Val_\mX \to \Val_\mZ$ reconstructs factors from variables using the underlying mechanisms $\mF$:
\begin{align*}
    \forall \mbx\in \Val_\mX,\ \forall Z\in \mZ, \quad (\tau(\mbx))_Z = F_Z((\mbx)_{\Pa_Z}).
\end{align*}
\parhead{Intervention map.}
The intervention map $\omega$ translates hard interventions on $\mX$ to parametric ones on $\mZ$:
$$
\begin{aligned}
    \omega \ :\ \mI_L\subseteq\Hard_\mX &\to \mI_H \subseteq \Para_\mZ \\
    i_L &\mapsto i_H
\end{aligned}
$$
More specifically, $i_L\in\mI_L$ is a hard intervention characterized by its support $\mbI_L\subseteq\mX$ and the replacement of some low-level mechanisms by some constant assignment:
$$\forall X\in\mbI_L, \ F_X^L(\cdot) \leftarrow C_X\in\Val_X$$
In the corresponding high-level system, $i_H\in\mI_H$ is a parametric intervention characterized by its support
$$\mbI_H := \bigcup_{X\in\mbI_L}\Ch_{X}$$ and the replacement of some high-level mechanisms: $\forall Z\in\mbI_H, \ F_Z^H \leftarrow I_Z^H\langle F_Z^H\rangle$, where for any $\mbz\in\Val_\mZ$ and $\mbu\in\Val_\mU$,
\begin{align*}
    I_Z^H\langle F_Z^H\rangle(\mbz_{\Pa_Z^H}, \mbu) := F_Z\Big(\{F^*_X(\mbz_{\Pa_X},\mbu)\}_{X\in \Pa_Z}\Big)
\end{align*}
where $F_X^*=\begin{cases} C_X & X\in\mbI_L \\ F_X & X\notin\mbI_L\end{cases}
$.

\begin{example}
    Continuing the previous example, a hard intervention $i\in\mI_L: X_4\leftarrow C_4$ will modify $F^H_{Z_3}$ by replacing it with $F^H_{Z_3} = F_{Z_3}(C_4, \ F_{X_6}(Z_2))$.
\end{example}

In words, $\mI_H$ contains the parametric interventions that are realizable given the current underlying mechanisms $\mF$ and all hard interventions $\mI_L$ allowed in $\mX$\footnote{A careful reader will see that our definition of $\mI_L$ considers that low-level hard interventions only disconnect edges $Z\to X$ from the perspective of the f-SCM $\mM$, not $X\to Z$ ones.}.

\begin{theorem}[Uniform transformation]\label{th:probabilistic_exact_transfo}
With $\tau$ and $\omega$ defined as above, $\mH$ is a uniform $(\tau - \omega)$-transformation of $\mL$.
\end{theorem}
The proof, available in Appendix \ref{sec:proofs_theory_th1}, essentially verifies the points of Definition \ref{def:proba_exact_transfo} and \ref{def:uniform_transfo} one by one. The commutativity of the diagram is proven by induction, from root to leaf nodes, and relies importantly on the fact that context is shared between $\mH$ and $\mL$.

\subsection{Identifiability of the f-DAG}
Factors $\mZ$ are unobserved latents, hence the question of whether they are recoverable from data.
Indeed, without further assumptions, one could imagine splitting a given factor $Z$ into two factors without affecting any observed variable.
First, we focus on the graphical structure and study the identifiability of the adjacency matrix $A$ and of its factorization $Q\diamond R^\top$ up to simple reparameterizations.

\begin{assumption}\label{assu:causal_sufficiency_markov_faithfulness}
    The low-level SCM $\mL$ verifies causal sufficiency, Markov property, and faithfulness\footnote{These assumptions are required for the identifiability statements, but not for Theorem \ref{th:probabilistic_exact_transfo}.}.
\end{assumption}

\begin{assumption}\label{assu:two_parents_per_latent}
    Each factor $Z\in \mZ$ has at least two parents in $G$.
\end{assumption}

Under Assumption \ref{assu:causal_sufficiency_markov_faithfulness} and \ref{assu:two_parents_per_latent}, \citet{lopezLargeScaleDifferentiableCausal2022} show that the low-rank assumption on the adjacency matrix $A$ over variables $\mX$ reduces the Markov equivalence class (MEC) over the graph $G_L$, thus making it often identifiable.

\begin{lemma}[Identifiability of $G_L$, from \citet{lopezLargeScaleDifferentiableCausal2022}]\label{lem:GL_identifiability}
    Under Assumption \ref{assu:two_parents_per_latent}, the MEC of $G_L$ reduces to a singleton.
\end{lemma}

Starting from this point, we propose conditions to ensure that the full f-DAG $G$, including the latent factors, is identifiable up to permutation of indices. We present identifiability results for the mechanisms in the next section.

The key extra assumption we introduce involves the existence of ``anchors'' for each latent factor. In topic modeling, anchors are words that are specific to a given topic: they only belong to one cluster, and no other \citep{arora2012learningtopicmodels}. Inspired by this idea, we define anchors for factors.

\begin{definition}[Anchors]\label{def:anchors}
For $X\in \mX$ and $Z\in \mZ$, we say that $X$ is an \emph{in-anchor} of $Z$ if $\Ch_X = \{Z\}$, i.e., the only child of $X$ is $Z$ in $G$. Analogously, $X$ is an \emph{out-anchor} of $Z$ if $\Pa_X = \{Z\}$, i.e., the only parent of $X$ is $Z$ in $G$.
\end{definition}

In matrix terms, this means that for any factor $Z$ with index $k$, $X$ is an in-anchor of $Z$ if its index $i$ is such that $Q_{ik}=1$ and $\forall l\ne k, Q_{il}=0$, i.e., the sum (in $\N$) of the row $i$ is 1 exactly. We can interpret out-anchors similarly on $R$.

\begin{assumption}\label{assu:in_out_anchor}
    Each factor $Z\in\mZ$ has at least one in-anchor and one out-anchor.
\end{assumption}
In high dimensional systems, Assumption \ref{assu:in_out_anchor} is fairly reasonable:
canonical marker genes are genes that may uniquely identify a given cell type \citep{zeisel2015cell}, and some regions of the brain correspond to ``specialized functional components'' \citep{sporns2016modular} where it is reasonable that some neurons in these specialized regions are anchors for the whole region.
Similarly, when solving some task, some neurons of a large neural network may only contribute to one high-level concept, while most of them contribute to many.

In particular, this assumption is softer than other causal abstraction frameworks which usually assume \emph{constructive abstractions}, which implies that $\mX$ can be partitioned into cells \citep{chalupkaMultiLevelCauseEffectSystems2015, geigerCausalAbstractionsNeural2021}, i.e., all parents of a latent $Z$ are in-anchors, and all its children are out-anchors.

Fix an adjacency matrix $A\in\{0,1\}^{n\times n}$ with $\rank(A)=m\le n$\footnote{The Boolean rank of $A$ is not easy to compute if one had to do it in practice, it is NP-complete \citep{orlin1977contentment}, but is uniquely defined by only looking at $A$.}.
By definition of the Boolean rank, there exists a decomposition of $m$ terms $A=Q \diamond R^\top$ with $Q,R\in \{0,1\}^{n\times m}$.

\begin{lemma}[Reading $Q,R$ in $A$]\label{lem:reading_UV_in_A_anchors}
    Under Assumption \ref{assu:in_out_anchor}, the columns of $Q$ and the rows of $R^\top$ are readable in $A$.
\end{lemma}
Intuitively, an anchor connects to exactly one latent factor, so its row/column in $A$ directly copies that factor's connectivity pattern.
For two vectors $x,y\in \R^n$, we write $x \le y$ if $\forall i\in[n], x_i\le y_i$.
\begin{definition}[Minimality]\label{def:minimality}
    A row $A_{i*}$ is \emph{minimal} if it is not null, and there is no other non-zero row $A_{h*}$ such that $A_{h*} \le A_{i*}$ and $A_{h*} \ne A_{i*}$. Minimal columns are defined analogously.
\end{definition}
This notion depends only on $A$, and characterizes the rows for which no other row is a strict subset of them.

\begin{lemma}[Minimality-anchor correspondence]\label{lem:minimal_rows}
    Under Assumption \ref{assu:in_out_anchor}, in-anchors correspond to the minimal rows of $A$, and out-anchors to the minimal columns of $A$.
\end{lemma}
Anchors appear as the minimal rows and columns in the adjacency matrix because their pure connection patterns cannot be formed by combining the connections of other factors.
\begin{theorem}[Identifiability of $(Q,R)$]\label{th:boolean_mtx_identifiability}
Under Assumption \ref{assu:in_out_anchor}, a Boolean factorization
$A=\bigvee_{k=1}^m Q_{*k}R_{k*}^\top$
is unique up to a simultaneous permutation of the $m$ factors.
\end{theorem}
Recovering the correct causal graph identifies the connections drawn by the latent factors, but does not give access to their actual values. The next section exposes what additional equivalence classes f-SCMs are subject to.

\subsection{Identifiability of the mechanisms}
Latent factors are only observed through their downstream effect on the variables, so the mechanisms of an f-SCM are identifiable only up to some reparameterization of the factor value spaces. This section makes this statement precise and gives simple conditions under which the admissible reparameterizations reduce to affine or monotone transformations.

\begin{definition}[$\mX$-equivalence]\label{def:x_equiv}
Two f-SCMs $\mM, \widetilde \mM$ are \emph{$\mX$-equivalent} if they induce the same low-level collapsed mechanisms on $\mX$, i.e.,
$\forall X\in\mX, \ F_X^L=\widetilde F_X^L$.
\end{definition}

\begin{lemma}[Reparameterization symmetry]\label{lem:reparam_mechanisms}
Let $\mM$ be an f-SCM and let $(\rho_Z)_{Z\in\mZ}$ be any family of bijections $\rho_Z:\Val_Z\to\Val_Z$. Define $\widetilde\mM$ over the same f-DAG by $\widetilde F_Z := \rho_Z^{-1}\circ F_Z$ and $\widetilde F_X\bigl(\mbz_{\Pa_X},\mbu\bigr) := F_X\bigl((\rho_Z(\mbz_Z))_{Z\in\Pa_X},\mbu\bigr)$.
Then $\widetilde\mM$ is $\mX$-equivalent to $\mM$.
\end{lemma}

Factors can be reparameterized by arbitrary bijections without affecting observational and interventional behavior on $\mX$. We propose assumptions to restrict this equivalence class.

\begin{definition}[Linear parent mechanisms]\label{def:linear_parent_mech}
An f-SCM has \emph{linear parent mechanisms} if for every $Z\in\mZ$ the mechanism $F_Z$ is affine: for all $\mbx\in\Val_\mX$,
\[
F_Z(\mbx_{\Pa_Z}) = b_Z + \sum_{X\in \Pa_Z} a_{ZX} \mbx_X,
\]
where $a_{ZX}, b_Z\in \R$, and with at least one nonzero coefficient per factor: $\forall Z\in \mZ, \ \exists X\in \Pa_Z, \ a_{ZX}\neq 0$.
\end{definition}

\begin{theorem}[Linear parent identifiability]\label{th:lin_parents_identifiability}
Let $\mM,\widetilde \mM$ be two $\mX$-equivalent f-SCMs sharing the same f-DAG with out-anchors for all factors, linear parent mechanisms and real supports.
Then for every $Z\in\mZ$ there exists a unique bijection $\rho_Z$ such that $\tilde F_Z = \rho_Z^{-1}\circ F_Z$, and it is affine: $\rho_Z : \mbz_Z \mapsto \alpha_Z \mbz_Z + \beta_Z$ with $\alpha_Z\ne 0$

\end{theorem}
We show in Appendix \ref{sec:proofs_theory_th3} similar results when child mechanisms are affine.

These results motivate the metrics that should be used to evaluate the method in practice. In particular, using metrics invariant to linear reparameterization to match factors $Z$ against external hypotheses only makes sense if we are willing to assume linearity of the mechanisms.

\section{Practical implementation}
We now specify the learning objective for our \emph{unsupervised causal abstraction discovery} method (UCAD), and show additional connections with causal abstractions. Our practical implementation draws inspiration from low-rank causal discovery, with specific additional constraints. On synthetic experiments, we show that our algorithm recovers ground-truth latent factors, up to the expected equivalence classes predicted by theory.

\subsection{Learning objective}
We extend the low-rank causal discovery objective of \citet{lopezLargeScaleDifferentiableCausal2022} with specific additional constraints. We parameterize separately the conditional distributions with $\Theta$ and the causal graph with $\Phi$. The constrained optimization problem is the following:
\begin{align*}
\max_{\Phi, \Theta} S(\Phi, \Theta) \text{ such that } \begin{cases}
    \mC^{\text{DAG}}(\Phi)=0\\
    \mC^{\text{anchors}}(\Phi)=0\\
    \mC^{\text{sparsity}}(\Phi)\le \text{tol}_s
\end{cases}
\end{align*}

\parhead{Objective function.}
\begin{align*}
    S(\Phi, \Theta) = \E_{\substack{G'\sim G(\Phi) \\ i\sim\PP}} \E_{\mbx\sim p^{(i)}_{\text{data}}} \sum_{X\notin \mbI} \log p_{\Theta}^X(\mbx_X \mid \mbx_{\Pa_X^{G'}})
\end{align*}
with $G(\Phi)$ the causal graph model, $\PP$ is the empirical distribution over intervention regimes $\mI_L$, $p^{(i)}_{\text{data}}$ the empirical distribution over observed variables $\mX$ under intervention $i$, and $p_{\Theta}^X$ is a density model for the conditional distribution of $X\in\mX$ conditioned on its parents $\Pa_X^{G'}$ in the sampled graph $G'\sim G(\Phi)$.

Our graph model $G(\Phi)$ is parametrized with a Gumbel-softmax on a matrix $W\in\{-1,0,1\}^{n\times m}$ such that $Q_{ij}=\vone\{W_{ij}=1\}$ and $R_{ij}=\vone\{W_{ij}=-1\}$. This way, $Q_{ij}$ and $R_{ij}$ cannot be simultaneously equal to $1$, and no self-loop is introduced.

Our conditional densities $p_{\Theta}^X$, which aim to recover the mechanisms $F_X^L$, are parameterized such that for each factor $Z$ indexed by $k\in[m]$ in $Q$ and each variable $X$ indexed by $j\in[n]$ in $R$, we have:
\begin{align*}
    \mbz_Z = F_Z^\Theta(Q_{:k}\mbx) ; \qquad
    \mbx_X \sim \text{Normal}(F_X^\Theta(R_{j:}\mbz), \ \sigma_j^2)
\end{align*}
where $\sigma_j^2>0$, and $F_Z^\Theta, F_X^\Theta$ are linear models or MLP in practice. In other words, we assume additive Gaussian noise on variables, while we do not introduce noise on factors.

\parhead{Acyclicity constraint.}
$\mC^{\text{DAG}}$ enforces that the final graph is indeed a DAG. Several formulations have been proposed in the literature, which we describe in Appendix \ref{sec:implem_details}. Additionally, we propose a constraint formulation for the case where variables $X\in\mX$ are ``layered'', as in a neural network where a neuron $X_1$ cannot causally influence another neuron $X_2$ if both are in the same layer or if $X_2$ precedes $X_1$ in the computational graph.
In other words, we restrict the search space to block upper-triangular adjacency matrices.
For a low-rank matrix $A = Q \diamond R^\top$, this translates into the per-factor condition:
\[
\forall k\in[m],\qquad \max_{i:Q_{ik}=1}\ell(i)<\min_{j:R_{jk}=1}\ell(j),
\]
i.e., every parent selected by column $Q_{:k}$ lies in a strictly earlier layer than every child selected by column $R_{:k}$. This translates into a differentiable constraint that can be efficiently computed in $O(nm+mL)$ time, with $L$ the number of layers, applied on $G(\Phi)$:
\begin{align*}
    \mC^{\text{DAG}}(\Phi) := \sum_{\underset{l_p\ge l_c}{k}} \left[\sum_{i:\ell(i)=l_p} Q(\Phi)_{ik}\right]  \left[\sum_{j:\ell(j)=l_c} R(\Phi)_{jk} \right]
\end{align*}

\parhead{Anchor constraint.}
As discussed previously, the Boolean factorization of the adjacency matrix is not unique unless additional assumptions are made, e.g. Assumption \ref{assu:in_out_anchor}. For our model to find such solutions, we consider $\mC^{\text{anchor}}$, applied to $G(\Phi)$.

A factor $Z$ indexed by $k\in[m]$ has an in-anchor if and only if there exists a row $i\in[n]$ such that $Q_{ik}=1$ and $\forall\,\ell\neq k:Q_{i\ell}=0$.
Equivalently,
$$
\forall k\in[m], \quad \max_{i} Q_{ik}\bigwedge_{\ell\neq k}(1-Q_{i\ell}) = 1
$$
We have a similar formulation for out-anchors.
This formulation is made differentiable using classic optimization tricks, to finally obtain $\mC^{\text{anchor}}$, as described in Appendix \ref{sec:implem_details}.

\paragraph{Sparsity constraint}
As in many causal discovery algorithms, we impose some sparsity constraint on the learned graph, in the form of an $L_1$ constraint
\begin{align*}
    \mC^{\text{sparsity}}(\Phi) := \frac{\sum_{ij}|Q(\Phi)_{ij}| + |R(\Phi)_{ij}|}{mn}
\end{align*}
One can adjust the tolerance level $\text{tol}_s$ depending on prior knowledge of the structure to be learned.

\subsection{Link with causal abstractions}
Exact transformation is a binary notion, while it is intuitive to have a gradient in how close $\mL$ and $\mH$ are. Following \citet{beckers2019approximatecausalabstraction}, we introduce a soft notion that captures this intuition:
\begin{definition}[Approximate transformation]\label{def:approx_transfo}
    Let $(\mL,\Pr)$ and $(\mH,\Pr)$ be two probabilistic causal models sharing their context $\mU$.
    Let $\tau : \Val_\mX \to \Val_\mZ$ and $\omega : \mI_L \to \mI_H$ be two partial surjective functions where $\omega$ is order preserving.

    Let $d_H$ a distance on $\mZ$, $\PP$ a probability distribution over $\mI_L$, and $\eta\ge0$ a scalar.
    $\mH$ is a $\eta$-$(\tau-\omega)$-approximate transformation of $\mL$ if
    \begin{align*}
        \E_{\substack{i\sim\PP \\ \mbu\sim\Pr}} \left[d_H \left(\tau\big(\Solve(\mL^{i};\mbu)\big), \Solve\big(\mH^{\omega(i)};\mbu\big)\right)\right] \le \eta
    \end{align*}
\end{definition}

Note that other choices of statistics than $\E$ can be made depending on the use-case, and that we specialized the definition to match our case where $\mU$ is shared by $\mL$ and $\mH$.
Taking $\eta=0$ recovers the definition of exact transformation.

In practice, $\hat\tau, \hat\omega$, and $\mH_{\Phi,\Theta}$ are learned, and we wish that
$$\hat \Delta := \E_{\substack{i\sim\PP \\ \mbu\sim\Pr}} \left[d_H \left(\hat\tau\big(\Solve(\mL^{i};\mbu)\big), \Solve\big(\mH^{\hat\omega(i)}_{\Phi,\Theta};\mbu\big)\right)\right]$$
diminishes as the model is trained.

\begin{theorem}[Approximate transformation bound]\label{th:approx_transfo_bound}
    If all mechanisms $F_Z^\Theta,F_X^\Theta$ are Lipschitz, then $\hat \Delta$ is linearly bounded by $S(\Phi,\Theta)$.
\end{theorem}
In words, under regularity conditions on the learned mechanisms, the learned high-level model $\mH_{\Phi,\Theta}$ is an $\eta$-$(\hat\tau-\hat\omega)$-approximate transformation of the ground-truth low-level model $\mL$, and $\eta$ diminishes as a linear function of the learning objective diminishes.

\subsection{Practical implementation}
We implement\footnote{See \url{https://github.com/theosaulus/ucad}.} this constrained optimization problem using PyTorch \citep{ansel2024pytorch} and Cooper \citep{gallegoPosada2025cooper}.

\subsection{Empirical results}
\parhead{Synthetic.}
We validate UCAD on synthetic f-SCMs with setups ranging from linear noiseless mechanisms with layered structure to nonlinear mechanisms with Gaussian noise on variables $X$ on non-layered DAGs. In all settings, ground-truth f-DAGs are generated with enforced anchors and models are trained on a mix of observational and interventional regimes. See Appendix \ref{sec:experiments} for details.

For each setup, we generate $5$ DAGs, on which we train $60$ UCAD models with varying hyperparameters and seed. For each DAG, we select the configuration with lowest validation negative log-likelihood (NLL) among those that reached constraint feasibility.

Since we aim to assess the recovery of latent factors, we pay particular attention to the mean correlation coefficient (MCC), which finds the alignment with highest absolute Pearson correlation between the learned and ground-truth factors. When dealing with arbitrary non-linear transformations, we consider MCC-RDC, which is based on the randomized dependence coefficient \citep{lopezpaz2013randomizeddependencecoefficient}. Both metrics are invariant to permutation of the latents.

In the case where we expect the f-SCM to be recoverable up to linear reparametrization (i.e., when $F_X$ and/or $F_Z$ are affine), a perfect MCC-Pearson of $1$ is theoretically attainable.
The performance of our model is fairly invariant to the number of variables ($16, 64$ or $128$) and factors ($3, 6$ or $8$), to whether the DAG has a layered structure or not, and to whether additive Gaussian noise is added on the variables. When averaging over all these cases (13 experiments) we measure an average MCC (Pearson) of $0.83$. Importantly, the variability we observe between cases does not seem to depend on whether both $F_X,F_Z$ are linear or only one.

We report MCC for a set of experiments in Table~\ref{tab:synthetic_results_mcc_mini}. In addition, we report MCC after applying random orthogonal transformations to both ground-truth factors and learned factors, to contextualize the results. MCC is equal to $1$ when the latent factors are perfectly recovered, up to linear reparametrization and permutation of indices, and decreases when the factors are mixed. This comparison shows that our method recovers latents that are not only  related to the ground-truth factors, but are also aligned with their coordinates.

\begin{table}[t]
\centering
\caption{Test MCC (Pearson) assessing reconstruction of ground-truth latent factors when $F_X, \hat F_X$ are linear and $F_Z, \hat F_Z$ are non-linear (mean $\pm$ std over $5$ DAGs, of the UCAD model achieving the lowest validation feasible objective). See Table \ref{tab:synthetic_results_mcc} for more details.}
\label{tab:synthetic_results_mcc_mini}
\resizebox{\linewidth}{!}{
\centering
\begin{tabular}{cccccc}
\toprule
Noisy & Layered & $n$-$m$ & Learned & {\begin{tabular}[c]{@{}c@{}}Ground-truth\\ rotated\end{tabular}}& {\begin{tabular}[c]{@{}c@{}}Learned\\ rotated\end{tabular}}\\
\hline
\multirow{2}{*}{$\times$} & \multirow{2}{*}{$\checkmark$} & 16-3  & $0.80\pm 0.16$ & $0.70 \pm 0.03$ & $0.65 \pm 0.11$ \\
                        &                             & 64-6  & $0.93\pm 0.08$ & $0.69 \pm 0.03$ & $0.64 \pm 0.04$ \\
\hline
\multirow{2}{*}{$\times$} & \multirow{2}{*}{$\times$}   & 16-3 & $0.73\pm 0.13$ & $0.68\pm 0.09$ &  $0.62 \pm 0.09$\\
                        &                           & 64-6 & $0.82\pm 0.03$ & $0.64 \pm 0.05$ & $0.58 \pm 0.04$ \\
\hline
\multirow{2}{*}{$\checkmark$} & \multirow{2}{*}{$\times$} & 16-3 & $0.82\pm 0.15$ & $0.68\pm0.05$ & $0.61 \pm 0.10$ \\
                            &                         & 64-6 & $0.78\pm 0.11$ & $0.67\pm0.03$ & $0.59 \pm 0.06$\\
\bottomrule
\end{tabular}
}
\end{table}

In the case where both $F_X,F_Z$ are non-linear, theory predicts that can only recover the factors up to an arbitrary bijective transformation, which results in a drop of MCC-Pearson ($0.63$). The MCC-RDC score remains high, at $0.79$ in the 64 variables case, and $0.70$ in the 16 variables case.

These experiments tend to indicate that when assumptions are satisfied, it is possible to train a model that is capable to recover the ground truth factors, following theory.
We still note a significant variability between each example, both in terms of performance and in which hyperparameters end up reaching the lowest validation NLL; algorithmic improvements and better search may help future developments.

\parhead{Neural network interpretability.}
To test UCAD on a real DNN, we consider the problem of divisibility by 6. Solving this problem is intuitively performed by checking divisibility by 2 (by looking at the last digit) and by 3 (by checking whether the sum of all digits is divisible by 3, recursively).

Let the input be $n\in\{0,\dots,N\}$ represented in base $10$ as $(d_0,\dots,d_{K-1})\in\{0,\dots,9\}^K$, and let the output be $y=\vone[6\mid n]\in\{0,1\}$. We consider the concepts:
\begin{align*}
c_2=\vone[2\mid n],\qquad c_3=\vone[3\mid n],\qquad c_6=c_2\wedge c_3.
\end{align*}
To match the intuitive process of computing $c_3$, we also consider an alternative concept $\tilde c_3=\sum_{k}d_k$.
This problem possesses natural ``negative controls'' that are irrelevant: divisibility by an unrelated prime (e.g. $5, 7, 11$).

We train five MLPs with $3$ hidden layers of size $32$ to solve this task for digits between 0 and 1 million, until convergence. Then, we gather the $96$ activations for unseen numbers and train our UCAD model with $3$ factors and report reconstruction metrics in Table \ref{tab:div6_results_mini}.

\begin{table}[htb]
\centering
\caption{Test metrics assessing reconstruction of positive and negative concepts by learned factors for the divisibility-by-6 task (mean $\pm$ std over 5 trained MLPs, of the UCAD model achieving lowest validation feasible objective). See Table \ref{tab:div6_results} for more details.}
\label{tab:div6_results_mini}
\resizebox{0.85\linewidth}{!}{
\begin{tabular}{lcc}
\toprule
Metric &  \multicolumn{2}{c}{96 neurons, 3 factors}\\
\midrule
NLL & \multicolumn{2}{c}{$-0.30 \pm 0.30$} \\
DAG Violation & \multicolumn{2}{c}{$1.84\cdot 10^{-6} \pm 3.98\cdot 10^{-6}$} \\
\midrule
& Positive concepts & Negative controls \\
MCC (Pearson) & $0.41 \pm 0.09$ & $0.03 \pm 0.01$\\
MCC (RDC) & $0.72 \pm 0.05$ & $0.07 \pm 0.01$ \\
\bottomrule
\end{tabular}
}
\end{table}

We measure a stronger reconstruction of hypothesized concepts than with negative controls (average MCC-RDC of $0.72$ vs. $0.07$ over 5 experiments). In particular, one of the factors has an average Pearson correlation of $0.80$ with $c_2$, which is a strong signal that it is consistently implemented by the network.
Reasonably high MCC-RDC coefficients can be observed for $c_6$ and $c_3$ ($0.55$ and $0.32$), especially compared to negative controls which are not reconstructed as well (at most $0.14$).
These results remain difficult to interpret, because a low agreement with some concept could sometimes be an indication that the DNN implements a different strategy to solve the problem, for example a heuristic that works only up to 1 million.

\section{Discussion and future work}

In this paper, we demonstrated that it is possible to discover causal abstractions directly from low-level measurements without relying on an expert to propose a candidate high-level model. By leveraging factor directed acyclic graphs and introducing the anchor assumption, we proved that low-rank causal discovery methods can identify high-level latent factors up to simple reparameterization. We translated these theoretical guarantees into a practical learning objective and empirically tested it on both synthetic datasets and a small neural network.

While our theoretical results provide a foundation for UCAD, they rely on knowing the true number of latent factors. A direction for future work is understanding the behavior of the learning objective when this number is misspecified.

Scaling our approach to larger and more complex systems is another direction. Our current experiments on the divisibility task highlight the difficulty of interpreting the heuristics learned by deep neural networks. Applying it to larger models in language, vision, or to noisy scientific data, will require finding scalable optimization procedures.

In particular, the hope of a causal abstraction in mechanistic interpretability lies in its ability to reliably control a system's output. Future research should focus on steering experiments, to test whether intervening on the unsupervised discovered high-level factors can reliably alter the behavior of DNNs on complex tasks.

\newpage
\begin{acknowledgements}

    Théo Saulus thanks Damien Ferbach for discussions on the algebraic structure over high-level interventions, Juan Ramirez for advices in implementing the constrained optimization problem, and Julien Boussard for general proofreading and comments.

    This research was partially supported by the NSERC Discovery Grant RGPIN-2023-04869 and by the Canada CIFAR AI Chair Program. Simon Lacoste-Julien is a CIFAR Associate Fellow in the Learning in Machines \& Brains program.

    This research was enabled in part by compute resources provided by Mila (\url{mila.quebec}).
\end{acknowledgements}

\bibliography{references}

\newpage
\onecolumn
\title{Unsupervised Causal Abstractions Discovery \\(Supplementary Material)}
\appendix

\section*{Appendix}
\begingroup
\makeatletter
\renewcommand*\l@section[2]{%
  \ifnum \c@tocdepth >\z@
    \addpenalty\@secpenalty
    \addvspace{0.1em \@plus\p@}%
    \setlength\@tempdima{1.5em}%
    \begingroup
      \parindent \z@ \rightskip \@pnumwidth
      \parfillskip -\@pnumwidth
      \leavevmode \bfseries
      \advance\leftskip\@tempdima
      \hskip -\leftskip
      #1\nobreak\hfil \nobreak\hb@xt@\@pnumwidth{\hss #2}\par
    \endgroup
  \fi}
\makeatother
\renewcommand{\contentsname}{Contents}
\tableofcontents
\endgroup
\newpage

\section{Additional figure}
We illustrate Boolean matrix factorization:
\[A_{ij}=1
\quad\iff\quad \exists k\in[m], \ Q_{ik}=1 \text{ and } R^\top_{kj}=1
\quad\iff\quad \exists k\in[m], X_i \rightarrow Z_k \rightarrow X_j\]
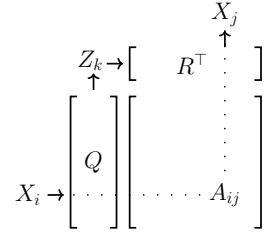
\begin{figure}[h]
    \centering
    \resizebox{0.2\columnwidth}{!}{
        \begin{tikzpicture}[
    thick_bracket/.style={line width=1.5pt},
    dot_line/.style={line width=1.5pt, line cap=round, dash pattern=on 0pt off 12pt}
]

\draw[thick_bracket] (-0.5, 2) -- (-0.7, 2) -- (-0.7, -2) -- (-0.5, -2);
\draw[thick_bracket] (0.5, 2) -- (0.7, 2) -- (0.7, -2) -- (0.5, -2);
\node at (0, 0) {\huge $Q$};

\draw[thick_bracket] (1.3, 2) -- (1.1, 2) -- (1.1, -2) -- (1.3, -2);
\draw[thick_bracket] (4.9, 2) -- (5.1, 2) -- (5.1, -2) -- (4.9, -2);

\draw[thick_bracket] (1.3, 3.5) -- (1.1, 3.5) -- (1.1, 2.5) -- (1.3, 2.5);
\draw[thick_bracket] (4.9, 3.5) -- (5.1, 3.5) -- (5.1, 2.5) -- (4.9, 2.5);
\node at (3, 3) {\huge $R^\top$};

\draw[dot_line] (-0.6, -1.0) -- (3.5, -1.0);

\draw[dot_line] (4.0, 4.0) -- (4.0, -0.5);

\node at (4.0, -1.0) {\huge $A_{ij}$};

\node at (-2.0, -1.0) {\huge $X_i$};
\draw[->, line width=1.5pt] (-1.4, -1.0) -- (-0.9, -1.0);

\node at (-0.1, 3.1) {\huge $Z_k$};
\draw[->, line width=1.5pt] (0, 2.2) -- (0, 2.7);
\draw[->, line width=1.5pt] (0.4, 3) -- (0.9, 3);

\node at (4.0, 4.5) {\huge $X_j$};
\draw[->, line width=1.5pt] (4.0, 3.5) -- (4.0, 4.0);

\end{tikzpicture}
    }
    \caption{Illustration of Boolean matrix factorization.}
    \label{fig:QR_product_schema}
\end{figure}

\section{Proofs of the theoretical results}\label{sec:proofs_theory}
\subsection{Proof of Theorem \ref{th:probabilistic_exact_transfo}}\label{sec:proofs_theory_th1}

We will first prove two lemmas related to the algebraic structure of $\mI_L$ and $\mI_H$, before we prove the theorem.

We have previously described that low level hard interventions $a\in\mI_L$ are characterized by
\begin{itemize}
    \item their support $\mbI_L$
    \item for each $X\in\mbI_L$ a constant assignment $C_X$ by which $F_X$ is replaced.
\end{itemize}

At the high level, recall how we define the collapsed mechanisms for each $Z\in\mZ$:
\begin{align*}
    F^{H}_Z : \Val_{\Pa^H_Z} \times \Val_\mU &\to \Val_Z \\
    \mbz, \mbu &\mapsto F_Z\Big(\{F_X(\mbz_{\Pa_X}, \mbu)\}_{X\in \Pa_Z}\Big)
\end{align*}

We characterize interventions $\alpha\in\mI_H$ by
\begin{itemize}
    \item their support $\mbI_H$
    \item for each $Z\in\mbI_H$ a ``mechanistic support'' $\mbK_Z\subseteq\Pa_Z$ which describes what mechanisms $\{F_X\}_{X\in\Pa_Z}$ are replaced inside of each $F_Z^H$
    \item for each $Z\in\mbI_H$ and $X\in\mbK_Z$ a constant assignment $C_X$ by which $F_X$ is replaced.
\end{itemize}

We therefore define the following partial order relation over $\mI_H$:

\begin{lemma}[Partial order on $\mI_H$]\label{lem:partial_order_IH}
    For $\alpha,\beta\in\mI_H$, let $\preceq_H$ be the relation defined as:
    \begin{align*}
        \alpha\preceq_H \beta  \iff \begin{cases}
            \mbI^\alpha_H \subseteq \mbI^\beta_H \\
            \forall Z\in \mbI^\alpha_H, \ \mbK^\alpha_Z \subseteq\mbK^\beta_Z \\
            \forall Z\in \mbI^\alpha_H, \ \forall X\in\mbK^\alpha_Z, \ C^\alpha_X = C^\beta_X
        \end{cases}
    \end{align*}
    $\preceq_H$ is a partial order.
\end{lemma}
\begin{proof}
Reflexivity is immediate.

For transitivity, consider $\alpha,\beta,\gamma\in\mI_H$ such that $\alpha \preceq_H \beta$ and $\beta \preceq_H\gamma$. We have:
\begin{itemize}
    \item $\mbI^\alpha_H \subseteq \mbI^\beta_H \subseteq\mbI^\gamma_H$
    \item for any $Z\in\mbI^\alpha_H$, we have $\mbK^\alpha_Z \subseteq \mbK^\beta_Z$, and for any $Z'\in\mbI^\beta_H$, we have $\mbK^\beta_{Z'} \subseteq \mbK^\gamma_{Z'}$. So in particular,
    \begin{align*}
        \forall Z\in\mbI^\alpha_H, \quad \mbK^\alpha_Z \subseteq \mbK^\gamma_Z
    \end{align*}
    \item for any $Z\in \mbI^\alpha_H$ and $X\in\mbK^\alpha_Z$, we have $C^\alpha_X = C^\beta_X$, and for any $Z'\in \mbI^\beta_H$ and $X'\in\mbK^\beta_{Z'}$, we have $C^\beta_{X'} = C^\gamma_{X'}$. So in particular,
    \begin{align*}
        \forall Z\in \mbI^\alpha_H, \ \forall X\in\mbK^\alpha_Z, \ C^\alpha_X = C^\gamma_X
    \end{align*}
\end{itemize}

For antisymmetry, consider $\alpha,\beta\in\mI_H$ such that $\alpha\preceq_H\beta$ and $\beta\preceq_H\alpha$. We have:
\begin{itemize}
    \item $\mbI^\alpha_H \subseteq \mbI^\beta_H \subseteq\mbI^\alpha_H$ so $\mbI^\beta_H =\mbI^\alpha_H$
    \item for any $Z\in\mbI^\alpha_H$, we have $\mbK^\alpha_Z \subseteq \mbK^\beta_Z$, and for any $Z'\in\mbI^\beta_H$, we have $\mbK^\beta_{Z'} \subseteq \mbK^\alpha_{Z'}$. So in particular, \begin{align*}
        \forall Z\in\mbI^\alpha_H, \ \mbK^\alpha_Z = \mbK^\beta_Z
    \end{align*}
    \item for any $Z\in \mbI^\alpha_H$ and $X\in\mbK^\alpha_Z$, we have $C^\alpha_X = C^\beta_X$, and for any $Z'\in \mbI^\beta_H$ and $X'\in\mbK^\beta_{Z'}$, we have $C^\beta_{X'} = C^\alpha_{X'}$. So in particular,
    \begin{align*}
        \forall Z\in \mbI^\alpha_H, \ \forall X\in\mbK^\alpha_Z, \ C^\alpha_X = C^\beta_X
    \end{align*}
\end{itemize}
Finally, $\alpha,\beta$ describe the exact same high-level intervention.

This concludes the proof.
\end{proof}

\begin{lemma}\label{lem:properties_omeaga_IH}
    $\mI_L$ and $\mI_H$ can be equipped with partial orders $\preceq_L$ and $\preceq_H$, and $\omega$ is order preserving with respect to them.
\end{lemma}
\begin{proof}
Let's address each point.

\paragraph{Defining $\preceq_L$}
We choose what is often referred to as the ``natural'' partial order \citep{beckersAbstractingCausalModels2019}: for $a,b\in\mI_L$,
\begin{align*}
    a\preceq_L b  \iff \begin{cases}
        \mbI^a_L \subseteq \mbI^b_L \\
        \forall X\in\mbI^a_L, \ C^a_X = C^b_X
    \end{cases}
\end{align*}

\paragraph{Defining $\preceq_H$}
We choose $\preceq_H$ as introduced in Lemma \ref{lem:partial_order_IH}.

\paragraph{$\omega$ is order preserving}
Let $a,b\in\mI_L$ such that $a\preceq_L b$. In other words, we have:
$\mbI^a_L \subseteq \mbI^b_L$ and $\forall X\in\mbI^a_L, \ C^a_X = C^b_X$

For all $Z\in\mbI_H^{\omega(a)} = \bigcup_{X\in\mbI^a_L} \Ch_X$, the intervention $\omega(a)$ sets the internal mechanisms $F_X$ to a $C_X$ for all $X\in\mbI^a_L\cap\Pa_Z$.

Since $\mbI^a_L \subseteq \mbI^b_L$, we obtain:
\begin{align*}
    \mbI_H^{\omega(a)} = \bigcup_{X\in\mbI^a_L} \subseteq \bigcup_{X\in\mbI^b_L} = \mbI_H^{\omega(b)}
\end{align*}

Additionally, we get:
\begin{align*}
    \forall Z\in\mbI^{\omega(a)}_H, \quad \mbI^a_L\cap\Pa_Z \subseteq \mbI^b_L\cap\Pa_Z
\end{align*}

Then, since we have $\forall X\in\mbI^a_L, \ C^a_X = C^b_X$, we get:
\begin{align*}
    \forall Z\in\mbI^{\omega(a)}_H, \ \forall X\in\mbI^a_L\cap\Pa_Z, \quad C^a_X = C^b_X
\end{align*}

Putting things together:
\begin{align*}
    \begin{cases}
        \mbI_H^{\omega(a)} \subseteq \mbI_H^{\omega(b)} \\
        \forall Z\in\mbI^{\omega(a)}_H, \quad \mbI^a_L\cap\Pa_Z \subseteq \mbI^b_L\cap\Pa_Z \\
        \forall Z\in\mbI^{\omega(a)}_H, \quad \forall X\in\mbI^a_L\cap\Pa_Z, \quad C^a_X = C^b_X
    \end{cases}
\end{align*}
That is to say, $\omega(a)\preceq_H\omega(b)$.

\end{proof}

\begin{customthm}{\ref{th:probabilistic_exact_transfo}}[Uniform transformation]
    With $\tau$ and $\omega$ defined above, $\mH$ is a uniform $(\tau-\omega)$-transformation of $\mL$.
\end{customthm}
\begin{proof}
We verify one by one the points of Definition \ref{def:proba_exact_transfo}.

\paragraph{Basic properties}
$\mI_L$ and $\mI_H$ are sets of interventions, and  $(\mI_L, \preceq_L)$ and $(\mI_H, \preceq_H)$ define partially ordered sets, by Lemma \ref{lem:properties_omeaga_IH}.

$G_L$ and $G_H$ are acyclic, so $\mL^i$ and $\mH^{\omega(i)}$ have unique solutions under the $\Solve$ operator for a given context $\mbu\in\Val_\mU$.

$\tau$ is surjective if we restrict its codomain to its image on the solution of $\mL$ under $\mI_L$.

$\omega$ is surjective: by construction, $\mI_H$ contains exactly the high-level parametric interventions induced by the allowed low-level hard interventions in $\mL$.

$\omega$ is order preserving by Lemma \ref{lem:properties_omeaga_IH}.

\paragraph{Commutativity of the diagram at fixed context}
To prove the commutativity of the diagram, we will first show that commutativity holds for a fixed context $\mbu\in\Val_\mU$, before considering a probability distribution on it.
Let $i\in \mI_L$ with support $\mbI_L\subseteq\mX$ and let $\mbu\in\Val_\mU$. Let $\mbx=\Solve(\mL^i; \mbu)\in \Val_\mX$. Let's prove by induction that $\tau(\mbx)$ equals $\mbz = \Solve(\mH^{\omega(i)};\mbu)$.
\begin{itemize}
    \item Base case. Let's fix $Z$ a source node in $G_H$. This means that $\Pa_Z^H=\varnothing$ and that $\forall X\in\Pa_Z,\  \Pa_X=\varnothing$. On the one hand, we have:
    $$(\mbz)_Z = I_Z^{H}(\varnothing, \mbu) = F_Z\Big(\{F^*_X(\varnothing, \mbu)\}_{X\in \Pa_Z}\Big), \quad \text{where} \ F_X^*=\begin{cases} C_X & X\in \mbI_L \\ F_X & X\notin \mbI_L \end{cases}$$
    On the other hand, by definition,
    $$(\tau(\mbx))_Z = F_Z\Big((\mbx)_{\Pa_Z}\Big) = F_Z\Big(\{F^*_X(\varnothing, \mbu)\}_{X\in \Pa_Z}\Big)$$
    Hence $(\tau(\mbx))_Z=(\mbz)_Z$.
    \item Inductive step. Let $Z\in \mZ$. Suppose that for every ancestor $Z'$ of $Z$ in $G_H$, we have $(\tau(\mbx))_{Z'}=(\mbz)_{Z'}$. Let's prove $(\tau(\mbx))_Z = (\mbz)_Z$. On the one hand, we have:
    \begin{align*}
        \tau(\mbx)_Z
        &= F_Z(\mbx_{\Pa_Z}) && \text{Definition of $\tau$}\\
        &= F_Z\Big(\{F^*_X(F_{Z'}(\mbx_{\Pa_{Z'}})_{Z'\in\Pa_X}, \mbu)\}_{X\in \Pa_Z}\Big) && \mbx = \Solve(\mL^i;\mbu)\\
        &= F_Z\Big(\{F^*_X(\tau(\mbx)_{\Pa_X}, \mbu)\}_{X\in \Pa_Z}\Big) && \text{Definition of $\tau$}\\
        &= F_Z\Big(\{F^*_X(\mbz_{\Pa_X}, \mbu)\}_{X\in \Pa_Z}\Big) && \text{Induction hypothesis}
    \end{align*}
    On the other hand, we have:
    \begin{align*}
        \mbz_Z
        &= I^{H}_Z(\mbz_{Pa_Z}, \mbu) && \mbz= \Solve(\mH^{\omega(i)};\mbu)\\
        &= F_Z\Big(\{F^*_X(\mbz_{\Pa_X}, \mbu)\}_{X\in \Pa_Z}\Big) && \text{Definition of $\omega$}
    \end{align*}
    Thus, $\tau(\mbx)_Z = \mbz_Z$, which concludes the induction.
\end{itemize}

\paragraph{Probabilistic commutativity of the diagram}
Now, keep $i\in\mI_L$ and let $\mbu\sim \Pr$. For any $\mbz\in\Val_\mZ$, we have:
\begin{align*}
\Pr_{\mH^{\omega(i)}}(\mbz)
&= \Pr \left(\{\mbu : \Solve(\mH^{\omega(i)};\mbu)=\mbz\}\right) && \text{Definition of $\Pr_{\mH^{\omega(i)}}$} \\
&= \Pr \left(\{\mbu : \tau(\Solve(\mL^{i};\mbu))=\mbz\}\right) && \text{Result above}\\
&= \Pr \left(\{\mbu, \mbx : \Solve(\mL^{i};\mbu)=\mbx \text{ and } \tau(\mbx)=\mbz\}\right) \\
&= \Pr_{\mL^i} \left(\{\mbx : \tau(\mbx)=\mbz\}\right) && \text{Definition of $\Pr_{\mL^i}$}\\
&= \tau(\Pr_{\mL^i}) (\mbz) && \text{Pushforward}
\end{align*}

In other words, $\mH$ is an exact $(\tau-\omega)$transformation of $\mL$ in the sense of Definition \ref{def:proba_exact_transfo}.

For uniformity, given that context $\mU$ is shared for $\mL$ and $\mH$, for any distribution $\Pr_L$ on the low-level context $\mU$, we simply choose $\Pr_H=\Pr_L$ at the high level. The previous argument gives exactness for this pair, hence $\mH$ is a uniform $(\tau - \omega)$-transformation of $\mL$.
\end{proof}

\subsection{Discussion on the ordering of interventions}
When we apply two interventions $i_1,i_2$ on $\mM$, it might occur that some variables are affected by both $i_1$ and $i_2$. To determine which assignment prevails, \citet{geigerCausalAbstractionTheoretical2025} introduce the notion of intervention algebras, to fix the order in which the interventions are applied:

\begin{definition}[Intervention algebras]\label{def:intervention_algebra}
Let $\mI$ be a set of interventions allowed on $\mV$. We equip it with an associative composition operator $\circ:\mI\times\mI\to\mI$ and an identity element $\id$.
The pair $(\mI,\circ)$, together with $\id$, forms an \emph{intervention algebra} if the following properties are verified:
\begin{itemize}
\item the identity element $\id$ is defined such that $\mM^{id}=\mM$, i.e., no intervention performed.
\item the composition operator $\circ$ is such that \(\forall i_1,i_2\in\mI, \ \mM^{\,i_1\circ i_2}=(\mM^{\,i_1})^{\,i_2}\), and $\circ$ is commutative if $i_1$ and $i_2$ target different variables, and left-annihilative if they target the same:
\begin{align*}
(i_1 \circ i_2)_V =
    \begin{cases}
        (i_2)_V & \text{if } V \in \mbI_2 \\
        (i_1)_V & \text{if } V \in \mbI_1 \setminus \mbI_2 \\
        (\id)_V & \text{if } V \notin \mbI_1 \cup \mbI_2
    \end{cases}
\end{align*}
\end{itemize}
An induced partial order $\preceq$ is defined by $\circ$ such that $\forall i_1, i_2 \in \mI, \ i_1 \preceq i_2 \iff i_2 \circ i_1 = i_2$, i.e., the model $\mM^{i_2 \circ i_1}$ remains identical to $\mM^{i_2}$ if and only if $i_1$ agrees with the replacements performed by $i_2$.
\end{definition}

As shown by \citet{geigerCausalAbstractionTheoretical2025}, hard and soft interventions both can always be equipped with an intervention algebra. This is not the case of parametric interventions, which require more care.

Although the family of parametric intervention we consider might be relatable to intervention algebras, we stay closer to prior notions of exact transformation \citep{beckersAbstractingCausalModels2019, rubenstein2017causal} which only require the existence of partial orderings on $\mI_L,\mI_H$ with respect to which $\omega$ is order preserving.

\subsection{Proof of Theorem \ref{th:boolean_mtx_identifiability}}\label{sec:proofs_theory_th2}

\begin{customlem}{\ref{lem:GL_identifiability}}[Identifiability of $G_L$, adapted from \citet{lopezLargeScaleDifferentiableCausal2022}, Appendix B.2]
    If every factor $Z\in \mZ$ has at least two parents in $G$, then the MEC of $G_L$ reduces to a singleton.
\end{customlem}
\begin{proof}
    Let ($X_i, X_j$) be an unoriented edge in the essential graph of $G_L$. By definition of f-DAGs, there exists a factor $Z\in \mZ$ such that $(X_i, Z)$ and $(Z, X_j)$ are part of $G$. Because the edge is unoriented, it cannot be part of a v-structure in $G_L$.

    In the case that every factor $Z\in \mZ$ has at least two parents, there can be no unoriented edges, and therefore the graph is identifiable.
\end{proof}

Recall Definition \ref{def:anchors} and Assumption \ref{assu:in_out_anchor}.
Fix an adjacency matrix $A\in\{0,1\}^{n\times n}$ with $\rank(A)=m\le n$
\begin{lemma}[Rank of the decomposition]\label{lem:decomposition_rank}
For all $Q,R\in \{0,1\}^{n\times m}$ verifying Assumption \ref{assu:in_out_anchor}, we have $\rank(Q\diamond R^\top)=m$.
\end{lemma}

\begin{proof}
Let's denote for this proof $\widetilde A :=Q\diamond R^\top =\bigvee_{k=1}^{m} Q_{*k} \wedge R_{k*}^\top$. For each $k\in[m]$, we assume the existence of at least an in-anchor $i_k$ and an out-anchor $j_k$. This implies that $\widetilde A_{i_k j_k}=1$, and $\widetilde A_{i_k j_\ell}=0$ for $k\neq \ell$, because $Q_{i_k l}=0$ for $l\neq k$ and $R_{l j_\ell}=0$ for $l\neq \ell$ by the anchor conditions.

$A$ is defined here as a Boolean factorization with $m$ terms, hence $\mathrm{rank}(\widetilde A)\le m$.

Consider any rank-1 term $q r^\top$ used in a factorization of $\widetilde A$. For the sake of contradiction, assume that it covers both entries $(i_k,j_k)$ and $(i_\ell,j_\ell)$ for $k\neq \ell$, i.e., $q_{i_k}=q_{i_\ell}=1$ and $r_{j_k}=r_{j_\ell}=1$. This implies that $q_{i_k} r_{j_\ell}^\top = 1$, i.e., it also covers $(i_k,j_\ell)$, contradicting $\widetilde A_{i_k j_\ell}=0$. In other words, no single rank-1 term can cover two distinct in/out anchor pair. Thus, at least $m$ rank-1 terms are required, each of them being the only one to cover one specific coordinate, so $\mathrm{rank}(\widetilde A)\ge m$.

Finally, $\mathrm{rank}(\widetilde A) = m$.
\end{proof}

\begin{lemma}\label{lem:not_all_can_be_rpz}
    Not all $A$ of rank $m$ can be represented by $Q\diamond R^\top$ verifying Assumption \ref{assu:in_out_anchor}.
\end{lemma}
\begin{proof}
    Counter example:
    $$A = \mathbf{1}_{3\times 3} - I_3
    =
    \begin{bmatrix}
    0 & 1 & 1 \\
    1 & 0 & 1 \\
    1 & 1 & 0
    \end{bmatrix}$$
    cannot be represented by $Q,R$ with anchors. Here the rank is $m=n$, but it still illustrates the issue, just add null blocks with this matrix on the top right to get an example with $m<n$.
\end{proof}

Given this counter-example, we will assume in the following that $A$ is generated by a decomposition $Q\diamond R^\top$ that respects Assumption \ref{assu:in_out_anchor}.

\begin{customlem}{\ref{lem:reading_UV_in_A_anchors}}[Reading $Q,R$ in $A$]
    Under Assumption \ref{assu:in_out_anchor}, the columns of $Q$ and the rows of $R^\top$ are readable in $A$.
\end{customlem}

\begin{proof}
We assume that for every factor with index $k\in [m]$ there exists (at least) an in-anchor with index $i\in[n]$. Then one can read the value of each row of $R^\top$ on a row of $A$:
$$\forall j\in [n], \quad A_{ij} = \bigvee_l Q_{il} \wedge R_{lj} = Q_{ik}\wedge R_{kj} = R_{kj}$$
If there exist several in-anchors for a given factor, then this equality is true for all of them.
We get the same result for out-anchor and reading the columns of $Q$ on columns of $A$.

Assuming both in- and out-anchoring for all factors, we can read both $Q$ and $R$ in $A$.
\end{proof}

Now, let's establish that $(Q,R)$ are unique up to trivial permutations of factors.
For two vectors $x,y\in \R^n$, we write $x \le y$ if $\forall i\in[n], x_i\le y_i$.
Recall Definition \ref{def:minimality}

\begin{customlem}{\ref{lem:minimal_rows}}[Minimality-anchor correspondence]
    Under Assumption \ref{assu:in_out_anchor}, in-anchors correspond to the minimal rows of $A$, and out-anchors correspond to the minimal columns of $A$.
\end{customlem}
\begin{proof}
    $(\implies)$ By contraposition, let's prove that if $i$ is not an in-anchor, then $A_{i*}$ is not a minimal row.

    Let $G_i := \{\ell : Q_{i\ell}=1\}$. For $i\in[n]$ not an in-anchor, we have $|G_i|\ge 2$ by definition of an in-anchor. Let $k,l\in G_i,  k\ne l$. Let $j_l$ an out-anchor of $l$, and $i_k$ and in-anchor of $k$.

    We have $A_{i*}=\bigvee_{r\in G_i}R_{r*}$, so in particular $A_{i_k*}=R_{k*}\le A_{i*}.$
    At coordinate $j_l$ we have $R_{l j_l}=1$ and $R_{r j_l}=0$ for all $r\neq l$ (out-anchor of $l$), hence
    $A_{ij_l}=\bigvee_{r\in G_i} R_{r j_l}=R_{l j_l}=1$ while $A_{i_k j_l}=R_{k j_l}=0$.
    Therefore $A_{i_k*}<A_{i*}$, so $A_{i*}$ is not minimal.

    $(\impliedby)$ Let's prove that if $i$ is an in-anchor, then $A_{i*}$ is a minimal row.

    Let $i_k$ be an in-anchor of $k\in[m]$. Suppose that $\exists h\in[n], A_{h*} \le A_{i_k *} = R_{k*}$. We have, in general, $A_{h*}=\bigvee_r Q_{hr}\wedge R_{r*}$.

    If $\exists r\ne k$ such that $Q_{hr}=1$ then $R_{r j_r}=1 \ge R_{k j_r}=0$, which contradicts $A_{h*} \le R_{k*}$.

    Thus, $A_{h*}=Q_{hk} \wedge R_{k*}$. If $Q_{hk} = 0$, then $A_{h*} = (0,\ldots,0)$ and then $A_{i_k *}$ is minimal. If $Q_{hk} = 1$, then $A_{h*}=R_{k*}=A_{i_k*}$ and $A_{i_k *}$ is minimal.

    For the out-anchors, apply the same reasoning to $A^\top$.
\end{proof}

\begin{customthm}{\ref{th:boolean_mtx_identifiability}}[Identifiability of $(Q,R)$]
Under Assumption \ref{assu:in_out_anchor}, any Boolean factorization
$A=\bigvee_{k=1}^m Q_{*k}R_{k*}^\top$
is unique up to a simultaneous permutation of the $m$ factors.
\end{customthm}

\begin{proof}
By Lemma~\ref{lem:minimal_rows}, in-anchors correspond exactly to minimal rows, and each such minimal row equals one row of $R$.

Now take any other factorization $A=\bigvee_{r=1}^m \widetilde Q_{* r}\wedge \widetilde R_{r*}$ that satisfies Assumption \ref{assu:in_out_anchor}.
Fix a minimal row $A_{i*}=R_{k*}$. Because $A_{i*}$ is minimal, it cannot be the Boolean OR of two or more distinct rows among $\{\widetilde R_{r*}\}_{r=1}^m$. Hence there is a unique index $\phi(k)$ with $\widetilde R_{\phi(k)*}=R_{k*}$. This defines a bijection $k\mapsto \phi(k)$ matching the rows of $\widetilde R$ to those of $R$.

By Lemma~\ref{lem:minimal_rows}, out-anchors correspond exactly to minimal columns, and each such minimal column equals one column of $Q$. Therefore, in the alternative factorization there is a unique index $\phi'(k)$ with $\widetilde Q_{* \phi'(k)}=Q_{*k}$.

For a factor $k$, let $i_k$ be any in-anchor and $j_k$ any out-anchor. We have $A_{i_k j_k} = Q_{i_k k}\wedge R_{k j_k} = 1$ and $Q_{i_k \ell}\wedge R_{\ell j_k}=0, \forall\ell\neq k.$
In the alternative factorization we have
\[
A_{i_k j_k} = \bigvee_{r=1}^m \widetilde Q_{i_k r}\wedge \widetilde R_{r j_k}.
\]
Because $j_k$ is an out-anchor of $k$ and $\widetilde R_{\phi(k)*}=R_{k*}$, we get $\widetilde R_{\phi(k) j_k}=1$ while $\widetilde R_{r j_k}=0$ for $r\neq \phi(k)$.
Similarly, because $i_k$ is an in-anchor of $k$ and $\widetilde Q_{* \phi'(k)}=Q_{* k}$, we get $\widetilde Q_{i_k \phi'(k)}=1$ while $\widetilde Q_{i_k r}=0$ for $r\neq \phi'(k)$.

Hence the only index $r$ that can produce the $1$ at $(i_k,j_k)$ in the alternative factorization must satisfy $\phi'(k) = \phi(k)$. Therefore there is a single permutation $P$ of the factors, with
\(\widetilde R = P R\) and \(\widetilde Q = Q P^\top\). This proves uniqueness up to a simultaneous permutation of the factor indices.\end{proof}

Interestingly, we note that our assumption shares similarities with assumptions made in the Boolean matrix factorization literature. In particular, \citet{desantisFactorizationBinaryMatrices2021} obtain a characterization of the notion of ``free rank'' that can be related to anchors, and \citet{mironBooleanDecompositionBinary2021} derive a sufficient condition for the uniqueness of the factorization expressed in terms of vector support that also reminds of anchors.
More connections could potentially be drawn to obtain more general results.

\subsection{Towards a more flexible identifiability result}
We introduce a more flexible condition that does not require that each factor has a in- and out-anchor. We only show a property of this condition, not a full identifiability result.

\begin{definition}[$r$-disjunctedness]
    Columns of $Q$ are $r$-disjunct if
    \begin{align*}
        \forall k\in [m], \forall S\subseteq [m]\backslash \{k\} \quad\text{with}\quad |S|\le r, \ \exists j\in [n] \quad\text{such that}\quad Q_{jk} = 1 \quad\text{and}\quad \forall l\in S, \ Q_{jl}=0
    \end{align*}

    In words, every factor can be distinguished from the union of any combination of at most $r$ other factors by at least one parent variable.
\end{definition}

Example: the following matrix is 1-disjunct but not 2-disjunct:
$$\begin{bmatrix}
1&0&1&1\\
1&1&0&1\\
1&1&1&0\\
0&1&1&1\\
\end{bmatrix}$$
Indeed, for any $k \neq l$, the column $j=l+1\mod 4$ satisfies $Q_{12}=0$ and $Q_{kj}=1$.
$2$-disjunctness fails because every column of $Q$ contains exactly one zero, so no column vanishes on two distinct rows simultaneously. Hence for any $k$ and any pair $S=\{a,b\}\subseteq[m]\setminus\{k\}$, no column $j$ satisfies $Q_{kj}=1$ and $Q_{aj}=Q_{bj}=0$.

We can of course have a similar definition for the rows of $R^\top$.
Note that considering $r=m-1$ defines exactly the in-/out-anchors.

\begin{definition}[Row and column sparsity]
    The row-sparsity of $Q$ is denoted $s(Q) = \max_i \sum_k Q_{ik}$, and similar definition for the column-sparsity of $R^\top$.
\end{definition}

\begin{lemma}[Reading $Q,R$ from $A$]\label{lem:reading_UV_disjuncted_sparse}

Let's assume the following points:
\begin{itemize}
    \item each factor $Z$ has an in-anchor (resp. out-anchor)
    \item the columns of $R$ (resp. $Q$) are $r$-disjunct
    \item $m-1 \ge r \ge s(Q)$ (resp. $s(R)$)
\end{itemize}
Then $(Q,R)$ can be read in $A$.
\end{lemma}

\begin{proof}
\textit{Recovering $R$.} For all factors $Z$ with index $k$ there exists an in-anchor $X$ with index $i$, i.e. $Q_{ik}=1$ and $\forall l\ne k, Q_{il}=0$. We can thus read the values of $R$ directly in $A$ (see above).

\textit{$A_{i*}$ is the $\vee$ of the rows of $R$.} For all variable index $i\in[n]$, let's define $G_i = \{k : Q_{ik}=1\}$. By definition, $|G_i| \le s(Q)$. We have the following: $$A_{ij} = \bigvee_{l=1}^m (Q_{il} \wedge R_{lj}) = \bigvee_{k\in G_i} (1 \wedge R_{kj}) = \bigvee_{k\in G_i} R_{kj}$$ Written with row notations: $$A_{i*} = \bigvee_{k\in G_i} R_{k*}$$

\textit{Uniqueness}. For the sake of contradiction, suppose that this $\vee$-decomposition is not unique. Then there exists $T\ne S$ such that $|S|,|T|\le R$ and $A_{i*} = \bigvee_{l\in S} R_{l*} = \bigvee_{l\in T} R_{l*}$. Pick $h\in S\backslash T$. By $r$-disjunctness on the set $T$, there exists $j\in [n]$ such that $R_{hj}=1$ and $\forall l\in T, R_{lj}=0$. In particular, we have $\bigvee_{l\in T} R_{lj} = 0$. Therefore we have at the same time $\bigvee_{l\in T} R_{lj} = 0$ and $\bigvee_{l\in S} R_{lj} = 1$, and thus $\bigvee_{l\in S} R_{l*} \ne \bigvee_{l\in T} R_{l*}$. This is contradictory with the hypothesis we made, so the set $G$ is the unique $\vee$-decomposition of size at most $r$.

\textit{Recovering $Q$ with $S$.} Let's prove that $\forall i\in [n], \forall k\in[m], Q_{ik}=1 \iff R_{k*} \le A_{i*}$
\begin{itemize}
    \item If $Q_{ik} = 1$, then $k\in G_i$, and $A_{i*}=\bigvee_{l\in G_i}R_{l*}$ so in particular $R_{k*} \le A_{i*}$
    \item If $R_{k*} \le A_{i*}$, let's suppose for sake of contradiction that $Q_{ik}=0$, which implies $k\notin G_i$. By hypothesis, we have $R_{k*}\le A_{i*} = \bigvee_{l\in G_i}R_{l*}$ with $|G_i| \le R$. So there exists no $j$ such that $R_{kj} \ge \bigvee_{l\in G_i}R_{lj}$, which contradicts the $r$-disjunctedness. Therefore, $Q_{ik}=1$.
\end{itemize}

Finally, we have shown that there is a unique possibility for each entry of $Q$, knowing $A,R$.

A similar proof can be derived for the case where we have out-anchors for each factor and columns of $Q$ are $r$-disjunct.
\end{proof}

We suspect that an identifiability theorem can be derived from these weaker conditions and leave it for future work.

\subsection{Proof of Theorem \ref{th:lin_parents_identifiability}}\label{sec:proofs_theory_th3}

\begin{customdef}{\ref{def:x_equiv}}[$\mX$-equivalence]
Two f-SCMs $\mM, \widetilde \mM$ are \emph{$\mX$-equivalent} if they induce the same low-level collapsed mechanisms on $\mX$, i.e.,
$\forall X\in\mX, \ F_X^L=\widetilde F_X^L$.
\end{customdef}

\begin{customlem}{\ref{lem:reparam_mechanisms}}[Reparameterization symmetry]
Let $\mM$ be an f-SCM and let $(\rho_Z)_{Z\in\mZ}$ be any family of bijections $\rho_Z:\Val_Z\to\Val_Z$. Define $\widetilde\mM$ over the same f-DAG by $\widetilde F_Z := \rho_Z^{-1}\circ F_Z$ and $\widetilde F_X\bigl(\mbz_{\Pa_X},\mbu\bigr) := F_X\bigl((\rho_Z(\mbz_Z))_{Z\in\Pa_X},\mbu\bigr)$.
Then $\widetilde\mM$ is $\mX$-equivalent to $\mM$.
\end{customlem}
\begin{proof}
Let $\mbu\in\Val_\mU$ and $X\in \mX$. We can observe that for $\mbx\in\Val_{\Pa_X^L}$,
\begin{align*}
    \widetilde F_X^L(\mbx, \mbu)
    &= \widetilde F_X\left(\{\widetilde F_Z(\mbx_{\Pa_Z})\}_{Z\in\Pa_X},\mbu\right) \\
    &= F_X\left(\{\rho_Z(\rho_Z^{-1}(F_Z(\mbx_{\Pa_Z})))\}_{Z\in\Pa_X},\mbu\right) \\
    &= F_X\left(\{F_Z(\mbx_{\Pa_Z})\}_{Z\in\Pa_X},\mbu\right) \\
    &= F_X^L(\mbx,\mbu)
\end{align*}
Hence $F_X^L=\widetilde F_X^L$ for all $X\in\mX$.
\end{proof}

\begin{customdef}{\ref{def:linear_parent_mech}}[Linear parent mechanisms]
An f-SCM has \emph{linear parent mechanisms} if for every $Z\in\mZ$ the mechanism $F_Z$ is affine: for all $\mbx\in\Val_\mX$,
\[
F_Z(\mbx_{\Pa_Z}) = b_Z + \sum_{X\in \Pa_Z} a_{ZX} \mbx_X,
\]
where $a_{ZX}, b_Z\in \R$, and with at least one nonzero coefficient per factor: $\forall Z\in \mZ, \ \exists X\in \Pa_Z, \ a_{ZX}\neq 0$.
\end{customdef}

\begin{customthm}{\ref{th:lin_parents_identifiability}}[Linear parent identifiability]
Let $\mM,\widetilde \mM$ be two $\mX$-equivalent f-SCMs sharing the same f-DAG with out-anchors for all factors, linear parent mechanisms and real supports.
Then for every $Z\in\mZ$ there exists a unique bijection $\rho_Z$ such that $\tilde F_Z = \rho_Z^{-1}\circ F_Z$, and it is affine: $\rho_Z : \mbz_Z \mapsto \alpha_Z \mbz_Z + \beta_Z$ with $\alpha_Z\ne 0$
\end{customthm}
\begin{proof}
Fix $Z\in\mZ$ and let $X$ be its out-anchor. Since $\Pa_X=\{Z\}$, we have $\Pa_X^L=\Pa_Z$, and $\mX$-equivalence at $X$ gives for all $\mbx\in\Val_{\Pa_Z}$ and $\mbu\in\Val_\mU$,
\begin{align*}
F_X\bigl(F_Z(\mbx),\mbu\bigr)=\widetilde F_X\bigl(\widetilde F_Z(\mbx),\mbu\bigr).
\end{align*}

By linearity of the parent mechanisms, we can write
$F_Z(\mbx)=\mba^\top\mbx + b$ and $\widetilde F_Z(\mbx)=\tilde{\mba}^\top\mbx+\tilde b$,
with $\mba\neq 0$ and $\tilde{\mba}\neq 0$.

Suppose for the sake of contradiction that $\mba$ and $\tilde{\mba}$ are linearly independent. Then the affine map $T:\mbx\mapsto\bigl(\mba^\top\mbx+b,\ \tilde{\mba}^\top\mbx+\tilde b\bigr)$ is surjective from $\Val_{\Pa_Z}$ onto $\R^2$. Hence every pair $(s,\tilde s)\in\R^2$ is attained as $(F_Z(\mbx),\widetilde F_Z(\mbx))$ for some $\mbx\in\Val_{\Pa_Z}$, so the two arguments of $F_X$ and $\widetilde F_X$ range independently over $\R$. From the previous equation we get that for all $(s,\tilde s)\in\R^2$
\[
F_X(s,\mbu)=\widetilde F_X(\tilde s,\mbu).
\]
Fixing $s$ and varying $\tilde s$ shows that $\widetilde F_X(\cdot,\mbu)$ is constant, and thus $F_X(\cdot,\mbu)$ as well, contradicting the faithfulness assumption for the out-anchor. Hence $\tilde{\mba}=\lambda\,\mba$ for some $\lambda\in\R$, and $\lambda\neq 0$ since $\tilde{\mba}\neq 0$.

Now, define $\rho_Z^{-1}(s):=\lambda(s-b)+\tilde b$, an affine bijection of $\R$ since $\lambda\neq 0$. For every $\mbx\in\Val_{\Pa_Z}$,
\[
\rho_Z^{-1}\bigl(F_Z(\mbx)\bigr)
=\lambda\,\mba^\top\mbx+\tilde b
=\tilde{\mba}^\top\mbx+\tilde b
=\widetilde F_Z(\mbx),
\]
so $\widetilde F_Z=\rho_Z^{-1}\circ F_Z$, and $\rho_Z(z)=\lambda^{-1}(z-\tilde b)+b$ is affine with slope $\alpha_Z=\lambda^{-1}\neq 0$.

Since $\mba\neq 0$ and $\Val_X=\R$ for all $X\in\Pa_Z$, $F_Z$ is surjective on $\R$. Hence, any bijection $\rho$ with $\widetilde F_Z=\rho^{-1}\circ F_Z$ satisfies $\rho^{-1}=\rho_Z^{-1}$ on $\operatorname{range}(F_Z)=\R$, i.e., $\rho=\rho_Z$. In particular, every bijection realizing this relation is affine.
\end{proof}

\begin{definition}[Linear child mechanisms]\label{def:linear_child_mech}
An f-SCM has \emph{linear child mechanisms} if for every $X\in\mX$ the mechanism $F_X$ is affine in the factors: for all $\mbz\in\Val_\mZ$ and $\mbu\in\Val_\mU$,
\[
F_X(\mbz_{\Pa_X},\mbu) = b_X(\mbu) + \sum_{Z\in\Pa_X} a_{XZ}(\mbu)\,\mbz_Z,
\]
where $a_{XZ}(\mbu), b_X(\mbu)\in\R$, and with at least one child influenced by factor: $\forall Z\in\mZ,\ \exists X\in\Ch_Z,\ \exists\mbu\in\Val_\mU,\ a_{XZ}(\mbu)\neq 0$.
\end{definition}

\begin{theorem}[Linear child identifiability]\label{th:lin_child_identifiability}
Let $\mM,\widetilde\mM$ be two $\mX$-equivalent f-SCMs sharing the same f-DAG with out-anchors for all factors, linear child mechanisms and real supports. Then for every $Z\in\mZ$ there exists an affine bijection $\rho_Z:\mbz_Z\mapsto\alpha_Z\mbz_Z+\beta_Z$ with $\alpha_Z\neq 0$ such that
$\widetilde F_Z=\rho_Z^{-1}\circ F_Z$, and the coefficients $(\alpha_Z,\beta_Z)$
are uniquely determined.
\end{theorem}

\begin{proof}
Fix $Z\in\mZ$ and let $X$ be its out-anchor. Since $\Pa_X=\{Z\}$, we have $\Pa_X^L=\Pa_Z$, and $\mX$-equivalence at $X$ gives for all $\mbx\in\Val_{\Pa_Z}$ and $\mbu\in\Val_\mU$,
\[F_X\bigl(F_Z(\mbx),\mbu\bigr)=\widetilde F_X\bigl(\widetilde F_Z(\mbx),\mbu\bigr).\]
As $\Pa_X=\{Z\}$, the linear child mechanisms read $F_X(z,\mbu)=a(\mbu)\,z+b(\mbu)$ and $\widetilde F_X(z,\mbu)=\tilde a(\mbu)\,z+\tilde b(\mbu)$, writing $a(\mbu):=a_{XZ}(\mbu)$ and $\tilde a(\mbu):=\tilde a_{XZ}(\mbu)$.
Substituting, we get that for all $\mbx\in\Val_{\Pa_Z}$ and $\mbu\in\Val_\mU$
\begin{align*}
a(\mbu)F_Z(\mbx) + b(\mbu)=\tilde a(\mbu)\widetilde F_Z(\mbx) + \tilde b(\mbu).
\end{align*}

By faithfulness, the edges $\Pa_Z\to X$ in $G^L$ are not spurious, so  $\mbx\mapsto F_X^L(\mbx,\mbu)=a(\mbu)\,F_Z(\mbx)+b(\mbu)$ is
non-constant for some $\mbu'$. Since $X$ reaches $\Pa_Z$ only through $Z$,
this forces both $a(\mbu')\neq 0$ and $F_Z$ non-constant.

Fixing $\mbu=\mbu'$ in the previous equation and dividing by $a(\mbu)\neq 0$, we get for all $\mbx\in\Val_{\Pa_Z}$
\[
F_Z(\mbx)=\frac{\tilde a(\mbu)}{a(\mbu)}\widetilde F_Z(\mbx) +\frac{\tilde b(\mbu)-b(\mbu)}{a(\mbu)}
=:\mu\widetilde F_Z(\mbx)+\nu.
\]
Since $F_Z$ is non-constant, the right-hand side is also non-constant, hence $\mu\neq 0$. Therefore
\[
\widetilde F_Z=\tfrac{1}{\mu}(F_Z-\nu)=\alpha_Z F_Z+\beta_Z,
\qquad \alpha_Z:=\mu^{-1}\neq 0,\quad \beta_Z:=-\frac{\nu}{\mu}.
\]
Define $\rho_Z^{-1}:w\mapsto\alpha_Z w+\beta_Z$, an affine bijection of $\R$ since $\alpha_Z\neq 0$. Then $\widetilde F_Z=\rho_Z^{-1}\circ F_Z$, and $\rho_Z:z\mapsto\alpha_Z^{-1}(z-\beta_Z)=\mu z+\nu$ is affine with slope $\mu\neq 0$.

As $F_Z$ is non-constant, it ranges at least two distinct points, and an affine map is determined by its values at two points. Hence the relation $\widetilde F_Z=\rho_Z^{-1}\circ F_Z$ fixes $\rho_Z^{-1}$, and therefore the coefficients
$(\alpha_Z,\beta_Z)$, uniquely.

\end{proof}

\subsection{Proof of Theorem \ref{th:approx_transfo_bound}}\label{sec:proofs_theory_th4}
\begin{customthm}{\ref{th:approx_transfo_bound}}[Approximate transformation bound]
    If all mechanisms $F_Z^\Theta,F_X^\Theta$ are Lipschitz, then $\hat \Delta$ is linearly bounded by $S(\Phi,\Theta)$.
\end{customthm}

\begin{proof}
We instantiate the learned state map componentwise by the learned parent mechanisms:
\[
\forall \mbx\in\Val_{\mX},\  \forall Z\in\mZ,\quad (\hat\tau(\mbx))_Z := F_Z^\Theta(\mbx_{\Pa_Z})
\]
Consider the high-level model to be the learned one \(\mH_{\Phi,\Theta}\), with collapsed mechanisms given by
\[
F_Z^{H,\Theta}(\mbz_{\Pa_Z^H},\mbu) := F_Z^\Theta\Big(\{F_X^\Theta(\mbz_{\Pa_X},\mbu)\}_{X\in\Pa_Z}\Big),
\]
and \(\mL_{\Phi,\Theta}\) the corresponding low-level model on \(\mX\).
Assume \(d_H(\mbz,\mbz')=\sum_{Z\in\mZ}\|\mbz_Z-\mbz'_Z\|^2\), i.e., Euclidean MSE.

Fix an intervention \(i\in\mI_L\) with support \(\mbI\subseteq\mX\) and a context \(\mbu\in\Val_{\mU}\).

Set $\mbx := \Solve(\mL^{i};\mbu)$, i.e., $\mbx \sim p_{\text{data}}^{(i)}$, and
\[
\hat\mbz := \hat\tau(\mbx),\qquad
\mbz := \Solve(\mH_{\Phi,\Theta}^{\omega(i)};\mbu).
\]
\paragraph{Pointwise propagation inequality.}
For each $X\in\mX$, define the (regime-dependent) prediction
\[
\hat x_X :=
\begin{cases}
F_X^\Theta(\hat\mbz_{\Pa_X},\mbu) & \text{if } X\notin \mbI\\
\mbx_X & \text{if } X\in \mbI
\end{cases}
\qquad
\tilde x_X :=
\begin{cases}
F_X^\Theta(\mbz_{\Pa_X},\mbu) & \text{if } X\notin \mbI\\
\mbx_X & \text{if } X\in \mbI
\end{cases}
\]
so that $\tilde x_X$ matches the value of $X$ used inside $\Solve(\mH^{\omega(i)};\mbu)$, and for intervened variables $X\in\mbI$ we have $\tilde x_X=\mbx_X$ by construction.
Define per-variable residuals and per-factor errors
\[
r_X := \mathbf 1\{X\notin\mbI\}\,\|\mbx_X-\hat x_X\|^2,
\qquad
\delta_Z := \|\hat\mbz_Z-\mbz_Z\|^2.
\]

Fix $Z\in\mZ$. Using $\hat\mbz_Z = F_Z^\Theta(\mbx_{\Pa_Z})$ and $\mbz_Z = F_Z^\Theta(\tilde\mbx_{\Pa_Z})$, and that $F_Z$ is $L_Z$-Lipschitz, we have
\begin{align*}
    \delta_Z &= \|F_Z^\Theta(\mbx_{\Pa_Z}) - F_Z^\Theta(\tilde\mbx_{\Pa_Z})\|^2 \\
    &\le L_Z^2\sum_{X\in\Pa_Z}\|\mbx_X-\tilde x_X\|^2.
\end{align*}
For each $X\in\Pa_Z$, apply $\|a-b\|^2\le 2\|a-c\|^2+2\|c-b\|^2$ with $a=\mbx_X$, $b=\tilde x_X$, $c=\hat x_X$:
\begin{align*}
\|\mbx_X-\tilde x_X\|^2
&\le 2\|\mbx_X-\hat x_X\|^2 + 2\|\hat x_X-\tilde x_X\|^2 \\
&= 2r_X + 2\|F_X^\Theta(\hat\mbz_{\Pa_X},\mbu) - F_X^\Theta(\mbz_{\Pa_X},\mbu)\|^2.
\end{align*}

If $F_X^\Theta$ is $\ell_X$-Lipschitz in its first argument (uniformly in $\mbu$), then
\begin{align*}
\|F_X^\Theta(\hat\mbz_{\Pa_X},\mbu) - F_X^\Theta(\mbz_{\Pa_X},\mbu)\|^2
&\le \ell_X^2\sum_{Z'\in\Pa_X}\|\hat\mbz_{Z'}-\mbz_{Z'}\|^2 \\
&= \ell_X^2\sum_{Z'\in\Pa_X}\delta_{Z'}.
\end{align*}

Combining the last three results yields, for every $Z\in\mZ$,
\[
\delta_Z \ \le \ \sum_{X\in\mX} \underbrace{\bigl(2L_Z^2 \ \mathbf 1\{X\in\Pa_Z\}\bigr)}_{=:M_{Z,X}}\ r_X
\ + \
\sum_{Z'\in\mZ}\sum_{X\in\mX}
\underbrace{\bigl(2L_Z^2\ \mathbf 1\{X\in\Pa_Z\}\bigr)}_{M_{Z,X}}
\underbrace{\bigl(\ell_X^2\ \mathbf 1\{Z'\in\Pa_X\}\bigr)}_{=:N_{X,Z'}}\ \delta_{Z'}
\]

In vector form:
\[
\delta \le M r + (MN)\delta,
\]
where $r=(r_X)_{X\in\mX}$ and $\delta=(\delta_Z)_{Z\in\mZ}$.

\paragraph{Expectation and nilpotency.}
Taking $\E_{\substack{i\sim\PP \\ \mbu\sim\Pr}}$ preserves the componentwise inequality, giving
\[
D \le MR + (MN)D,
\]
where $D_Z:=\E[\delta_Z]$ and $R_X:=\E[r_X]$.

Because the underlying f-DAG is acyclic, the induced factor-level graph is acyclic, by Lemma \ref{lem:acy_different_levels_lopez}.

The matrix $MN$ has support contained in this adjacency, hence is nilpotent: $(MN)^\gamma=0$ for some $\gamma\le |\mZ|=m$. Therefore
\[
(I-MN)^{-1} = \sum_{k=0}^{\gamma-1}(MN)^k
\]
is well-defined and entrywise nonnegative.

Left-multiplying $(I-MN)D\le MR$ by $(I-MN)^{-1}\ge 0$ preserves the order:
\begin{align*}
D &\le (I-MN)^{-1}MR \\
&= \sum_{k=0}^{\gamma-1}(MN)^kMR
\end{align*}

Summing coordinates and using any submultiplicative norm (e.g.\ $\|\cdot\|_1$) yields
\begin{align*}
    \hat\Delta
    &=\E\left[\sum_{Z\in\mZ}\delta_Z\right] \\
    & =\|D\|_1 \\
    & \le \Bigl(\|M\|_1\sum_{k=0}^{\gamma-1}\|MN\|_1^k\Bigr)\cdot\|R\|_1
\end{align*}

\paragraph{Relating $\|R\|_1$ to the score $S(\Phi,\Theta)$.}
Under the Gaussian conditional model used for $p_\Theta^X$ (with variance $\sigma_X^2$),
\[
\log p_\Theta^X(\mbx_X\mid \hat\mbz_{\Pa_X})
= -\frac{1}{2\sigma_X^2}\|\mbx_X-\hat x_X\|^2 - c_X,
\qquad
c_X:=\frac{1}{2}\log(2\pi\sigma_X^2),
\]
hence, taking $\E_{\substack{i\sim\PP \\ \mbu\sim\Pr}}$ and summing over $X\notin\mbI$,
\begin{align*}
S(\Phi,\Theta) = -\sum_{X\in\mX}\frac{1}{2\sigma_X^2}R_X - C_0,
\qquad
C_0 := \E_{i\sim\PP}\Big[\sum_{X\notin\mbI}c_X\Big].
\end{align*}

Let $\sigma_{\max}^2:=\max_X\sigma_X^2$. Then
\begin{align*}
    \|R\|_1
    &=\sum_X R_X \\
    &\le 2\sigma_{\max}^2 \sum_X \frac{1}{2\sigma_X^2}R_X \\
    &=2\sigma_{\max}^2\bigl(-S(\Phi,\Theta)-C_0\bigr)
\end{align*}
Combining with the previous bound on $\hat \Delta$ gives
\begin{align*}
\hat\Delta \le 2\sigma_{\max}^2\ \|M\|_1\sum_{k=0}^{\gamma-1}\|MN\|_1^k\ \bigl(-S(\Phi,\Theta)-C_0\bigr),
\end{align*}
which is an affine bound in $S(\Phi,\Theta)$.
\end{proof}

\section{Practical implementation details}\label{sec:implem_details}

\subsection{Acyclicity constraint formulations}

We implement the main differentiable formulations for acyclicity constraints, all of the form $h(A)=0 \iff G$ is a DAG. We consider $A\in \R^{n\times n}_{+}$, for which nilpotence is equivalent to DAGness of $G$.

\textit{Matrix exponential (tr-exp).} Introduced by \citet{zheng2018dags}, this formulation defines:
\[
h_{\text{exp}}(A) = \tr\bigl(e^{A}\bigr) - n
\]
where $A$ is the (possibly weighted) $n\times n$ adjacency matrix. Computing $e^A$ costs $O(n^3)$.

\textit{Matrix power.} This formulation uses:
\[
h_{\text{pow}}(A) = \tr\left[\left(I + \frac{A}{n}\right)^n\right] - n
\]
which avoids the matrix exponential but retains the same $O(n^3)$ cost due to the matrix power computation.

\textit{Spectral radius.}
Proposed by \citep{lee2019scaling}, the spectral radius formulation enforces:
\[
h_{\rho}(A) = \rho\bigl(A\bigr)
\]
where $\rho(\cdot)$ is the spectral radius (largest eigenvalue in modulus). The spectral radius is computed via power iteration with $T$ steps, costing $O(Tn^2)$ per gradient step. \citet{nazaret2024stabledifferentiablecausaldiscovery} showed that this formulation is numerically more stable than the tr-exp and matrix power approaches, which we experimentally verify.

\parhead{Reduction to the factor graph.}
Since $A = Q\diamond R^\top$ is a rank-$m$ Boolean matrix with $m < n$, a key structural property allows working with the much smaller $m\times m$ factor adjacency:
\begin{lemma}[\citet{lopezLargeScaleDifferentiableCausal2022}]\label{lem:acy_different_levels_lopez}
$G$ is acyclic $\iff$ $G_X$ is acyclic $\iff$ $G_Z$ is acyclic.
\end{lemma}
Consequently, any of the above formulations can be applied to the factor adjacency $B = R^\top \diamond Q\in\R^{m\times m}_+$ instead of the full variable adjacency $A\in\R^{n\times n}$. This reduces the tr-exp and matrix power costs from $O(n^3)$ to $O(m^3 + nm^2)$, and the spectral radius cost from $O(Tn^2)$ to $O(Tm^2 + nm^2)$, with memory dropping from $O(n^2)$ to $O(m^2)$.

\parhead{Spectral radius adapted for the bipartite structure.}
\citet{lopezLargeScaleDifferentiableCausal2022} further exploit the bipartite structure to avoid forming $B$ altogether. The bipartite adjacency can be written as the block matrix:
\[
M = \begin{pmatrix} 0 & Q \\ R^\top & 0 \end{pmatrix} \in \R^{(n+m)\times(n+m)}_+
\]
Power iteration on $M$ decomposes into alternating matrix-vector products with $Q$ and $R^\top$:
\begin{enumerate}
    \item Initialize left and right eigenvector buffers $x_v\in\R^n, x_f\in\R^m, y_v\in\R^n, y_f\in\R^m$.
    \item For $t=1,\ldots,T$: normalize $x_f \leftarrow R^\top x_v / \|R^\top x_v\|$, then $x_v \leftarrow Q\, x_f / \|Q\, x_f\|$; symmetrically for $y$.
    \item Compute the spectral radius estimate via the gradient formula:
    \[
    \rho = \frac{y_v^\top Q\, x_f + y_f^\top R^\top x_v}{y_f^\top x_f + y_v^\top x_v}
    \]
    with $\nabla_Q \rho = \frac{y_v x_f^\top}{y_f^\top x_f + y_v^\top x_v}$ and $\nabla_R \rho = \frac{x_v y_f^\top}{y_f^\top x_f + y_v^\top x_v}$.
\end{enumerate}
Each iteration costs $O(nm)$, yielding a total of $O(Tnm)$ per gradient step with $O(nm)$ memory, never forming any $n\times n$ or $m\times m$ matrix. The complexity of all approaches is summarized in Table~\ref{tab:dagness_constraint_complexity}.

\begin{table}[h]
\centering
\small
\begin{tabular}{L{2cm}|C{1.35cm}|C{2cm}C{2cm}C{2.3cm}|C{1.2cm}}
\hline
& \textbf{Matrix} & \textbf{Tr-exp} & \textbf{Matrix power} & \textbf{Spectral radius} & \textbf{Memory} \\
\hline
On $Q\diamond R^\top$
& $\R^{n\times n}_+$
& $O(n^3 + n^2 m)$
& $O(n^3 + n^2 m)$
& $O(Tn^2 + n^2 m)$
& $O(n^2)$ \\
\hline
On $R^\top\diamond Q$
& $\R^{m\times m}_+$
& $O(m^3 + n m^2)$
& $O(m^3 + n m^2)$
& $O(Tm^2 + n m^2)$
& $O(m^2)$ \\
\hline
Bipartite spectral radius
& $2 \times \R^{n\times m}_+$
& --
& --
& $O(Tnm)$
& $O(nm)$ \\
\hline
\end{tabular}
\caption{Per-gradient-step cost. Here $n$ is the number of observed variables, $m$ the number of factors, and $T$ the number of power-iteration steps for the spectral radius.}
\label{tab:dagness_constraint_complexity}
\end{table}

In practice, in accordance with the insights of \citet{nazaret2024stabledifferentiablecausaldiscovery}, we observed that spectral radius methods are less prone to numerical instabilities than other formulation. In particular, we found the bipartite spectral radius formulation to be the most stable and efficient, and is used in all our experiments on non-layered DAGs.

\parhead{Layered DAG constraint.}
When variables $X\in\mX$ are organized into $L$ known layers (e.g., neurons in successive layers of a neural network), we can replace the general acyclicity constraint with a cheap structural constraint that also enforces this ordering constraint. Let $\ell(i)\in\{0,\ldots,L-1\}$ denote the layer of variable $i$. For the graph to be acyclic, it suffices that every factor has all its parents in strictly earlier layers than all its children:
\[
\forall k\in[m],\quad \max_{i:Q_{ik}=1}\ell(i) < \min_{j:R_{jk}=1}\ell(j).
\]

To make this differentiable, we aggregate the relaxed assignment matrices $Q,R\in(0,1)^{n\times m}$ into per-layer masses for each factor $k$:
\[
q_{kl} := \sum_{i:\ell(i)=l} Q_{ik}, \qquad r_{kl} := \sum_{j:\ell(j)=l} R_{jk}.
\]
This aggregation is computed in $O(nm)$ time. We then compute for each factor $k$ a violation score that measures the total mass of parent-child pairs violating the layer ordering:
\[
b_k := \sum_{l_p \ge l_c} q_{k,l_p} \cdot r_{k,l_c}.
\]
This is evaluated efficiently using a cumulative sum: letting $\bar{r}_{k,l} := \sum_{l'\le l} r_{k,l'}$, we obtain $b_k = \sum_{l=0}^{L-1} q_{k,l}\cdot \bar{r}_{k,l}$ in $O(mL)$ time.
In the binary case, $b_k=0$ is equivalent to the strict layer ordering; in the relaxed case, $b_k$ continuously measures the degree of violation.

For scale stability, we optionally normalize by the total mass:
\[
\bar{b}_k := \frac{b_k}{\left(\sum_l q_{kl}\right)\left(\sum_l r_{kl}\right) + \epsilon}.
\]
The constraint is formulated as the inequality $\bar{b}_k - \varepsilon \le 0$ for each factor $k$, with tolerance $\varepsilon > 0$. The total cost is $O(nm + mL)$, i.e. $O(nm)$ since $L\le n$.

\subsection{Causal graph sampling}
During training, the graph $G(\Phi)$ is sampled stochastically at each forward pass. We parameterize the bipartite structure through a single logit tensor $W(\Phi)\in\R^{n\times m\times 3}$, where for each variable-factor pair $(i,k)$ the three categories encode: 0 for a parent edge ($Q_{ik}=1$, i.e., variable $i$ is a parent of factor $k$), 1 for child edge ($R_{ik}=1$, i.e., variable $i$ is a child of factor $k$), and 2 for no edge. Sampling is performed with the Gumbel-softmax trick \citep{jang2017categorical, maddison2017concrete}, using straight-through hard samples for the forward pass while allowing gradients to flow through the soft relaxation.

\subsection{Anchor constraint}

Recall that Assumption~\ref{assu:in_out_anchor} requires each factor $Z$ with index $k\in [m]$ to have at least one in-anchor and one out-anchor. We formalize this as a differentiable constraint.

For in-anchors, the score measuring how exclusively variable $i$ is assigned to factor $k$ as a parent is:
\[
s^{\text{in}}_{i\to k} = Q_{ik} \prod_{\ell\ne k}(1 - Q_{i\ell}).
\]
The in-anchor count for factor $k$ is then $c^{\text{in}}_k = \sum_{i=1}^n s^{\text{in}}_{i\to k}$, and symmetrically $c^{\text{out}}_k = \sum_{j=1}^n s^{\text{out}}_{j\to k}$ for out-anchors computed on $R$.

To avoid numerical underflow from the product of many terms close to~1, we compute the scores in log-space using the log-sum-exp decomposition:
\[
S_i := \sum_{\ell=1}^m \log(1 - Q_{i\ell} + \epsilon), \qquad
\log s^{\text{in}}_{i\to k} = \log(Q_{ik} + \epsilon) + S_i - \log(1 - Q_{ik} + \epsilon),
\]
where $\epsilon>0$ is a small constant for numerical stability. $S_i$ is the total log-complement sum across all factors, and subtracting the $k$-th term gives the product over $\ell\ne k$ without an explicit loop. The anchor score is recovered as $s^{\text{in}}_{i\to k} = \exp(\log s^{\text{in}}_{i\to k})$, and the count is $c^{\text{in}}_k = \sum_i s^{\text{in}}_{i\to k}$.

This computation runs in $O(nm)$ time for all factors simultaneously. The constraint is formulated as:
\[
\mC^{\text{anchor}}(\Phi) := \sum_{k=1}^m \left[\max\bigl(\tau - c^{\text{in}}_k, 0\bigr) + \max\bigl(\tau - c^{\text{out}}_k, 0\bigr)\right]
\]
where $\tau\le 1$ is a threshold (we use $\tau=0.99$). In practice, we impose per-factor inequality constraints $\tau - c^{\text{in}}_k \le 0$ and $\tau - c^{\text{out}}_k \le 0$ within the Augmented Lagrangian framework.

\section{Experiments}\label{sec:experiments}

\subsection{Implementation details}
Our implementation heavily modifies the open source implementation of \citet{lopezLargeScaleDifferentiableCausal2022}, in particular in the way we solve the constrained optimization problem.
Our model is implemented in PyTorch \citep{ansel2024pytorch}. Constrained optimization is handled by Cooper \citep{gallegoPosada2025cooper}, which provides an Augmented Lagrangian framework for differentiable constraints.

\parhead{Training loop.}
At each epoch, the model iterates over mini-batches of observational and interventional data. For each batch $(\mbx, \text{mask}, \text{regime})$:
\begin{enumerate}
    \item A graph $G'$ is sampled via Gumbel-softmax from the tensor $W(\Phi)$.
    \item The model computes predictions $\hat{\mbx}$ by passing inputs through the sampled factor graph (variables $\to$ factors via $F_Z^\Theta$, then factors $\to$ variables via $F_X^\Theta$).
    \item The data-fitting loss (negative log-likelihood) is computed over non-intervened variables.
    \item Constraint violations (DAG or layer ordering, anchors, $L_1$ sparsity) are evaluated.
    \item Cooper's Augmented Lagrangian solver performs alternating primal-dual updates: the primal optimizer (Adam \citep{kingma2014adam}) minimizes the Lagrangian, while the dual optimizer (SGD) updates the Lagrange multipliers.
\end{enumerate}

\parhead{Constrained optimization.}
All constraints use an Augmented Lagrangian formulation with inequality constraints $\mathcal{C}(\Phi)\le 0$.

\parhead{Hyperparameters.}
The primal learning rate is selected from $\{2\times 10^{-4}, 5\times 10^{-4}, 7\times 10^{-4}, 1\times 10^{-3}, 2\times 10^{-3}\}$ depending on the setting.
The dual learning rate is set as $\frac{1}{5}$ of the primal rate. The penalty coefficient $\rho_c$ is grown at a multiplicative rate of $1.003$ or $1.004$ per epoch.
Batch size is $512$ and training runs for up to $50\,000$ epochs with early stopping: patience of $2000$ epochs begins once all active constraints are satisfied. The $L_1$ sparsity tolerance is set to $\text{tol}_{s}=0.2$.

\parhead{Post-processing.}
After training, the continuous logits $W(\Phi)$ are discretized: for each variable-factor pair, we apply softmax over the three categories and assign the argmax direction (parent, child, or no edge), subject to a minimum confidence threshold on the total edge probability. Acyclicity of the recovered graph is then verified.

\parhead{Model selection and evaluation.}
For each experimental setting, we run a grid of hyperparameters (learning rate, penalty growth) over multiple random seeds and DAG instances. Model selection is performed on a validation set: among all runs that reach feasibility (all constraints satisfied), we select the one with the lowest validation negative log-likelihood. Test metrics are then computed on held-out observational data from the same DAG.

\subsection{Metrics}

We evaluate the quality of the learned causal abstraction by measuring factor recovery. Since factors are identifiable only up to reparameterization, we evaluate whether the learned factors $\hat{\mbz} = \hat\tau(\mbx)$ align with the ground-truth factors $\mbz$, using metrics invariant to the admissible symmetry class.
\begin{itemize}
    \item Mean correlation coefficient (MCC): finds the permutation alignment with highest absolute correlation between $\mbz$ and $\hat \mbz$ averaged over all factors. We report MCC with Pearson correlation (invariant to component-wise affine transformations), Spearman rank correlation (invariant to component-wise monotone transformations), and the Randomized Dependence Coefficient (RDC, \citet{lopezpaz2013randomizeddependencecoefficient}, invariant to component-wise nonlinear transformations).

    \item Explicitness ($R^2$): for each ground-truth factor, we fit a linear regression (or a small MLP) from all learned factors and report the variance-explained $R^2$, averaged over factors. $R^2$-linear measures linear decodability; $R^2$-MLP measures nonlinear decodability. Contrary to MCC which is sensitive to linear mixing by a matrix, $R^2$ is not.

    \item Disentanglement, Completeness, Informativeness (DCI) \citep{eastwood2018framework}: fits gradient boosted trees from learned factors to each ground-truth factor, extracts feature importances, and computes entropy-based scores. \emph{Disentanglement} measures whether each learned dimension captures at most one ground-truth factor; \emph{Completeness} measures whether each ground-truth factor is captured by at most one learned dimension; \emph{Informativeness} is the predictive $R^2$.
\end{itemize}

\subsection{Synthetic experiments}

\parhead{Data generation.}
We generate synthetic f-SCMs by first sampling a random bipartite DAG $G$ over $n$ variables and $m$ factors. Edge probabilities are set so that the graph is moderately sparse (probability of an edge set to $0.1$). When anchors are desired, the graph generator enforces Assumption~\ref{assu:in_out_anchor} by checking whether they are verified, otherwise the graph is resampled. When variables are organized into layers, factors are constrained to have parents in strictly earlier layers than children.

For each DAG, variable-to-factor mechanisms $F_Z$ are either linear (weights drawn uniformly from $[\pm0.25,\pm1]$) or nonlinear (two-layer neural networks with $\tanh$ activation and standard normal weights, i.e., each weight drawn i.i.d. from $\mathcal{N}(0,1)$).
Factor-to-variable mechanisms $F_X$ follow the same scheme. When the setting is noiseless, factors and variables are computed deterministically; otherwise, additive Gaussian noise with scale $0.4$ is added to variables (factors remain noiseless).

Data are generated under hard stochastic interventions: a certain amount of interventional regimes target $1$ to$3$ variables each, plus one observational regime. Each dataset contains $n=10\,000$ samples. We generate $5$ independent DAG instances per setting to account for graph variability.

\parhead{Evaluation.}
We evaluate UCAD across settings of variable difficulty (all with continuous variables and enforced anchors). We sweep across a few learning rates, dual learning rate, penalty multiplier, and seeds. In the following tables, the DAGness constraint violation may be negative because we target $\mC^{\text{DAG}}\le 10^{-6}$, and anything respecting the constraint will have a negative violation.

\begin{table*}[htb]
\centering
\caption{Summary of results: latent factors reconstruction (MCC-Pearson on test set, mean $\pm$ std, over $5$ DAGs, of the model achieving the lowest feasible objective after hyperparameter search on validation set).}
\label{tab:synthetic_results_mcc}
\centering
\begin{tabular}{|c|c|c|c|c|c|}
\hline
\multicolumn{2}{|c|}{Are $F_Z, \hat F_Z$ linear?} & $\checkmark$ & $\checkmark$ & $\times$ & $\times$ \\
\multicolumn{2}{|c|}{Are $F_X, \hat F_X$ linear?} & $\checkmark$ & $\times$ & $\checkmark$ & $\times$ \\
\hline
\multirow{3}{*}{noiseless, layered} & 16-3  & $0.88\pm 0.10$ &  & $0.80\pm 0.16$ &  \\
                                    & 64-6  & $0.80\pm 0.10$ &  & $0.93\pm 0.08$ &  \\
                                    & 128-8 & $0.93\pm 0.05$ &  &  &  \\
\hline
\multirow{2}{*}{noiseless, not layered} & 16-3  & $0.87\pm 0.16$ &  & $0.73\pm 0.13$ &  \\
                                        & 64-6  & $0.81\pm 0.06$ &  & $0.82\pm 0.03$ &  \\
\hline
\multirow{2}{*}{noisy, not layered} & 16-3  &  & $0.87\pm0.16$ & $0.82\pm 0.15$ & $0.64\pm0.04$ \\
                                    & 64-6  &  & $0.76\pm0.08$ & $0.78\pm 0.11$ & $0.62\pm0.09$ \\
\hline
\end{tabular}
\end{table*}

\vfill

\textbf{(a) Linear noiseless, layered} ($n\in\{16,64,128\}$, $m\in\{3,6,8\}$, $L=3$ layers). Both $F_Z$ and $F_X$ are linear, no noise is added to the variables, and variables are organized into layers (e.g., $[4,8,4]$ for $n=16$).
\begin{table}[H]
\centering
\caption{Aggregated test metrics on latent factors reconstruction (mean $\pm$ std, over $5$ DAGs, of the model achieving the lowest feasible objective among 80 configs).}
\label{tab:synthetic_results_a}
\begin{tabular}{lccc}
\toprule
 & 16 variables, 3 factors & 64 variables, 6 factors & 128 variables, 8 factors \\
 Metric & [4,8,4] & [16,32,16] & [32,64,32]\\
\midrule
MCC (Pearson) & $0.8826 \pm 0.1039$ & $0.7978 \pm 0.1033$ & $0.9294 \pm 0.0516$ \\
MCC (Spearman) & $0.8793 \pm 0.1069$ & $0.7932 \pm 0.1040$ & $0.9284 \pm 0.0516$ \\
MCC (RDC) & $0.8842 \pm 0.1025$ & $0.8376 \pm 0.0785$ & $0.9631 \pm 0.0414$ \\
\midrule
$R^2$ Linear & $1.0000 \pm 0.0000$ & $0.8424 \pm 0.1048$ & $0.9432 \pm 0.0488$ \\
$R^2$ MLP & $0.9973 \pm 0.0003$ & $0.8389 \pm 0.1028$ & $0.9347 \pm 0.0516$ \\
\midrule
DCI Info. & $0.9962 \pm 0.0048$ & $0.8279 \pm 0.1113$ & $0.9336 \pm 0.0546$ \\
DCI Disent. & $0.7638 \pm 0.1623$ & $0.6542 \pm 0.1544$ & $0.8774 \pm 0.0730$ \\
DCI Compl. & $0.8354 \pm 0.1051$ & $0.6873 \pm 0.1388$ & $0.8829 \pm 0.0679$ \\
\midrule
NLL & $0.4647 \pm 0.0023$ & $0.3360 \pm 0.3988$ & $0.0898 \pm 0.3083$ \\
DAG Violation & $-5.48\cdot 10^{-7} \pm 2.44\cdot 10^{7}$ & $-6.33\cdot 10^{-7} \pm 1.59\cdot 10^{-7}$ & $-6.18\cdot 10^{-7} \pm 1.44\cdot 10^{-7}$ \\
\bottomrule
\end{tabular}
\end{table}

\textbf{(b) Linear noiseless, not layered} ($n\in\{16,64\}$, $m\in\{3,6\}$). Same as (a) but without layer structure. We employ the spectral radius DAG constraint adapted for $Q\diamond R^\top$.
\begin{table}[H]
\centering
\caption{Aggregated test metrics on latent factors reconstruction (mean $\pm$ std, over $5$ DAGs, of the model achieving the lowest feasible objective among 80 configs).}
\label{tab:synthetic_results_b}
\begin{tabular}{lcc}
\toprule
 Metric & 16 variables, 3 factors & 64 variables, 6 factors \\
\midrule
MCC (Pearson) & $0.8681 \pm 0.1615$ & $0.8104 \pm 0.0577$ \\
MCC (Spearman) & $0.8689 \pm 0.1606$ & $0.8084 \pm 0.0577$ \\
MCC (RDC) & $0.8837 \pm 0.1425$ & $0.8289 \pm 0.0495$ \\
\midrule
$R^2$ Linear & $1.0000 \pm 0.0000$ & $0.8224 \pm 0.0638$ \\
$R^2$ MLP & $0.9971 \pm 0.0007$ & $0.8145 \pm 0.0639$ \\
\midrule
DCI Info. & $0.9993 \pm 0.0005$ & $0.8123 \pm 0.0664$ \\
DCI Disent. & $0.7481 \pm 0.2056$ & $0.7038 \pm 0.1081$ \\
DCI Compl. & $0.8791 \pm 0.1601$ & $0.7528 \pm 0.0816$ \\
\midrule
NLL & $-0.4133 \pm 0.2107$ & $0.2373 \pm 0.1827$ \\
DAG Violation & $-7.15\cdot 10^{-7} \pm 3.04\cdot 10^{-7}$ & $-9.17\cdot 10^{-7} \pm 1.67\cdot 10^{-7}$ \\
\bottomrule
\end{tabular}
\end{table}

\textbf{(c) Nonlinear $F_Z$, noiseless, layered} ($n\in\{16,64\}$, $m\in\{3,6\}$, $L=3$ layers). The learned model uses a 2-layer MLP for $F_Z^\Theta$ while $F_X^\Theta$ remains linear. No noise is added to variables.
\begin{table}[H]
\centering
\caption{Aggregated test metrics on latent factors reconstruction (mean $\pm$ std, over $5$ DAGs, of the model achieving the lowest feasible objective among 60 configs).}
\label{tab:synthetic_results_c}
\begin{tabular}{lcc}
\toprule
 Metric & 16 variables, 3 factors & 64 variables, 6 factors \\
\midrule
MCC (Pearson) & $0.8049 \pm 0.1594$ & $0.9317 \pm 0.0818$ \\
MCC (Spearman) & $0.8037 \pm 0.1603$ & $0.9310 \pm 0.0827$ \\
MCC (RDC) & $0.8249 \pm 0.1430$ & $0.9669 \pm 0.0482$ \\
\midrule
$R^2$ Linear & $0.8004 \pm 0.1630$  & $0.9379 \pm 0.0743$ \\
$R^2$ MLP & $0.8002 \pm 0.1617$ & $0.9352 \pm 0.0731$ \\
\midrule
DCI Info. & $0.7998 \pm 0.1633$ & $0.9347 \pm 0.0779$ \\
DCI Disent. & $0.6724 \pm 0.2686$ & $0.8985 \pm 0.1180$ \\
DCI Compl. & $0.7067 \pm 0.2431$ & $0.9162 \pm 0.0966$ \\
\midrule
NLL & $0.6791 \pm 0.2335$ & $0.4676 \pm 0.3000$ \\
DAG Violation & $-5.49\cdot 10^{-7} \pm 2.25\cdot 10^{-7}$ & $-6.35\cdot 10^{-7} \pm 1.72\cdot 10^{-7}$ \\
\bottomrule
\end{tabular}
\end{table}

\textbf{(d) Nonlinear $F_Z$, noiseless, not layered} ($n\in\{16,64\}$, $m\in\{3,6\}$). Same as (c) but without layer structure.
\begin{table}[H]
\centering
\caption{Aggregated test metrics on latent factors reconstruction (mean $\pm$ std, over $5$ DAGs, of the model achieving the lowest feasible objective among 60 configs).}
\label{tab:synthetic_results_d}
\begin{tabular}{lcc}
\toprule
 Metric & 16 variables, 3 factors & 64 variables, 6 factors \\
\midrule
MCC (Pearson) & $0.7303 \pm 0.1360$ & $0.8232 \pm 0.0335$ \\
MCC (Spearman) & $0.7424 \pm 0.1347$ & $0.8246 \pm 0.0348$ \\
MCC (RDC) & $0.8139 \pm 0.1449$ & $0.8649 \pm 0.0577$ \\
\midrule
$R^2$ Linear & $0.7886 \pm 0.1632$ & $0.8493 \pm 0.0534$ \\
$R^2$ MLP & $0.8015 \pm 0.1572$ & $0.8500 \pm 0.0575$ \\
\midrule
DCI Info. & $0.8141 \pm 0.1519$ & $0.8490 \pm 0.0605$ \\
DCI Disent. & $0.6043 \pm 0.2800$ & $0.7240 \pm 0.0643$ \\
DCI Compl. & $0.6990 \pm 0.3103$ & $0.7774 \pm 0.0528$ \\
\midrule
NLL & $0.6493 \pm 0.3245$ & $0.8417 \pm 0.2454$\\
DAG Violation & $-5.28\cdot 10^{-7} \pm 3.13\cdot 10^{-7}$ & $-6.18\cdot 10^{-7} \pm 2.57\cdot 10^{-7}$ \\
\bottomrule
\end{tabular}
\end{table}

\textbf{(e) Nonlinear $F_Z$, noisy, not layered} ($n\in\{16,64\}$, $m\in\{3,6\}$). Same as (d) but with additive Gaussian noise (scale $0.4$) on variables.
\begin{table}[H]
\centering
\caption{Aggregated test metrics on latent factors reconstruction (mean $\pm$ std, over $5$ DAGs, of the model achieving the lowest feasible objective among 60 configs).}
\label{tab:synthetic_results_e}
\begin{tabular}{lcc}
\toprule
 Metric & 16 variables, 3 factors & 64 variables, 6 factors \\
\midrule
MCC (Pearson) & $0.8219 \pm 0.1454$ & $0.7808 \pm 0.1056$ \\
MCC (Spearman) & $0.8212 \pm 0.1440$ & $0.7801 \pm 0.1056$ \\
MCC (RDC) & $0.8894 \pm 0.1030$ & $0.8027 \pm 0.0939$ \\
\midrule
$R^2$ Linear & $0.7948 \pm 0.1453$ & $0.7907 \pm 0.1268$ \\
$R^2$ MLP & $0.8069 \pm 0.1485$ & $0.7981 \pm 0.1241$ \\
\midrule
DCI Info. & $0.8138 \pm 0.1491$ & $0.7925 \pm 0.1279$ \\
DCI Disent. & $0.6951 \pm 0.1992$ & $0.6768 \pm 0.1407$ \\
DCI Compl. & $0.7851 \pm 0.1620$ & $0.7333 \pm 0.1111$ \\
\midrule
NLL & $1.3214 \pm 0.0647$ & $1.3704 \pm 0.0359$ \\
DAG Violation & $-8.28\cdot 10^{-7} \pm 1.43\cdot 10^{-7}$ & $-6.16\cdot 10^{-7} \pm 2.32\cdot 10^{-7}$ \\
\bottomrule
\end{tabular}
\end{table}

\textbf{(f) Nonlinear $F_X$, noisy, not layered} ($n\in\{16,64\}$, $m\in\{3,6\}$). The learned model uses a 2-layer MLP for $F_X^\Theta$ while $F_Z^\Theta$ remains linear.
\begin{table}[H]
\centering
\caption{Aggregated test metrics on latent factors reconstruction (mean $\pm$ std, over $5$ DAGs, of the model achieving the lowest feasible objective among 60 configs).}
\label{tab:synthetic_results_f}
\begin{tabular}{lcc}
\toprule
 Metric & 16 variables, 3 factors & 64 variables, 6 factors \\
\midrule
MCC (Pearson) & $0.8650 \pm 0.1565$ & $0.7616 \pm 0.0808$ \\
MCC (Spearman) & $0.8660 \pm 0.1563$ & $0.7725 \pm 0.0838$ \\
MCC (RDC) & $0.8791 \pm 0.1405$ & $0.8571 \pm 0.0410$ \\
\midrule
$R^2$ Linear & $0.8581 \pm 0.1570$ & $0.8624 \pm 0.0931$ \\
$R^2$ MLP & $0.8574 \pm 0.1553$ & $0.8752 \pm 0.0935$ \\
\midrule
DCI Info. & $0.8558 \pm 0.1560$ & $0.8346 \pm 0.0572$ \\
DCI Disent. & $0.7912 \pm 0.2474$ & $0.7101 \pm 0.0890$ \\
DCI Compl. & $0.8332 \pm 0.1979$ & $0.7893 \pm 0.0610$ \\
\midrule
NLL & $1.2294 \pm 0.0313$ & $1.2993 \pm 0.0531$ \\
DAG Violation & $-4.21\cdot 10^{-7} \pm 3.62\cdot 10^{-7}$  & $-6.39\cdot 10^{-7} \pm 1.68\cdot 10^{-7}$ \\
\bottomrule
\end{tabular}
\end{table}

\textbf{(g) Both $F_Z,F_X$ nonlinear, noisy, not layered} ($n\in\{16,64\}$, $m\in\{3,6\}$). Both $F_Z$ and $F_X$ are nonlinear MLPs. The learned model uses 2-layer MLPs for both $F_Z^\Theta$ and $F_X^\Theta$ (hidden dim $16$).
\begin{table}[H]
\centering
\caption{Aggregated test metrics on latent factors reconstruction (mean $\pm$ std, over $5$ DAGs, of the model achieving the lowest feasible objective among 60 configs).}
\label{tab:synthetic_results_g}
\begin{tabular}{lcc}
\toprule
 Metric & 16 variables, 3 factors & 64 variables, 6 factors \\
\midrule
MCC (Pearson) & $0.6366 \pm 0.0441$ & $0.6248 \pm 0.0892$ \\
MCC (Spearman) & $0.6247 \pm 0.0472$ & $0.6453 \pm 0.0936$ \\
MCC (RDC) & $0.6955 \pm 0.0338$ & $0.7916 \pm 0.0531$ \\
\midrule
$R^2$ Linear & $0.6237 \pm 0.0067$ & $0.6036 \pm 0.0695$ \\
$R^2$ MLP & $0.6524 \pm 0.0137$ & $0.6770 \pm 0.0647$ \\
\midrule
DCI Info. & $0.6566 \pm 0.0121$ & $0.6871 \pm 0.0623$ \\
DCI Disent. & $0.5131 \pm 0.0682$ & $0.5523 \pm 0.1130$ \\
DCI Compl. & $0.6043 \pm 0.1008$ & $0.6324 \pm 0.0958$ \\
\midrule
NLL & $1.2558 \pm 0.0622$ & $1.4137 \pm 0.0449$ \\
DAG Violation & $-2.46\cdot 10^{-7} \pm 4.86\cdot 10^{-7}$ & $-7.05\cdot 10^{-7} \pm 2.87\cdot 10^{-7}$ \\
\bottomrule
\end{tabular}
\end{table}

In the case where we expect the f-SCM to be recoverable up to linear reparametrization (i.e., when $F_X$ and/or $F_Z$ are affine), a perfect MCC-Pearson of $1$ is theoretically attainable.
We observe that the performance of our model is fairly invariant to the number of variables ($16, 64$ or $128$) and factors ($3, 6$ or $8$), to whether the DAG has a layered structure or not, and to whether additive Gaussian noise is added on the variables.
When averaging over all these cases (13 experiments) we measure an average MCC (Pearson) of $0.83$. Importantly, the variability we observe between all cases does not seem to depend on whether both $F_X,F_Z$ are linear or only one of them.

In the case where both $F_X,F_Z$ are non-linear, we can only recover the factors up to an arbitrary bijective transformation, which results in a drop of MCC-Pearson. The MCC-RDC score remains high, at $0.79$ in the 64 variables case, and $0.70$ in the 16 variables case.
We note however that DCI metrics and $R^2$-MLP, which are also metrics that should not be sensitive to non-linear transformations, are not matching this trend. Since the variance of MCC-RDC remains low, it does not look like a seed artefact, and experiments on more graphs are needed to understand where this discrepancy comes from.
\vfill

\paragraph{Ablations.}
We report additional experiments.

\textbf{(h) Both $F_Z,F_X$ are linear, noiseless, not layered} ($n=64$, $m=6$). The model is trained only on the observational regime, and tested on held-out observational samples.
\begin{table}[H]
\centering
\caption{Aggregated test metrics on latent factors reconstruction (mean $\pm$ std, over $5$ DAGs, of the model achieving the lowest feasible objective among 60 configs).}
\label{tab:synthetic_results_h}
\begin{tabular}{lc}
\toprule
Metric & 64 variables, 6 factors \\
\midrule
MCC (Pearson) & $0.8330 \pm 0.0751$ \\
MCC (Spearman) & $0.8300 \pm 0.0755$ \\
MCC (RDC) & $0.8617 \pm 0.0571$ \\
\midrule
$R^2$ Linear & $0.8874 \pm 0.0351$ \\
$R^2$ MLP & $0.8688 \pm 0.0613$ \\
\midrule
DCI Info. & $0.8624 \pm 0.0661$ \\
DCI Disent. & $0.7093 \pm 0.1152$ \\
DCI Compl. & $0.7423 \pm 0.0984$ \\
\midrule
NLL & $0.0988 \pm 0.4957$ \\
DAG Violation & $-8.31\cdot 10^{-7} \pm 1.51\cdot 10^{-7}$ \\
\bottomrule
\end{tabular}
\end{table}
We observe performances similar to the case where the model is also trained on interventional regimes. This setting may be simple, but still clearly hints at a good robustness of the model when only trained on observational data.

\vfill
\textbf{(i) Both $F_Z,F_X$ are linear, noiseless, not layered} ($n=64$, $m=6$). The model is trained on a subset of interventional regimes and tested on new ones. Following the code of \citet{lopezLargeScaleDifferentiableCausal2022}, each interventional regime has a fixed target set of 1 to 3 observed variables. Targeted variables are replaced by a new $\text{Normal}(0,1)$ draw, and non-target variables are simulated by the equations of the f-SCM. In total, 64 interventional regimes are created. Training samples are generated by the observational regime and 46 interventional regimes, validation is computed on 6 separate regimes, while test metrics are computed on 12 separate regimes.
\begin{table}[H]
\centering
\caption{Aggregated test metrics on latent factors reconstruction (mean $\pm$ std, over $5$ DAGs, of the model achieving the lowest feasible objective among 60 configs).}
\label{tab:synthetic_results_i}
\begin{tabular}{lc}
\toprule
Metric & 64 variables, 6 factors \\
\midrule
MCC (Pearson) & $0.8502 \pm 0.0987$ \\
MCC (Spearman) & $0.8471 \pm 0.0989$ \\
MCC (RDC) & $0.8894 \pm 0.0692$ \\
\midrule
$R^2$ Linear & $0.8715 \pm 0.0830$ \\
$R^2$ MLP & $0.8704 \pm 0.0886$ \\
\midrule
DCI Info. & $0.8648 \pm 0.0942$ \\
DCI Disent. & $0.7590 \pm 0.1489$ \\
DCI Compl. & $0.7798 \pm 0.1319$ \\
\midrule
NLL & $0.2530 \pm 0.3849$ \\
DAG Violation & $-7.36\cdot 10^{-7} \pm 2.04\cdot 10^{-7}$ \\
\bottomrule
\end{tabular}
\end{table}
We observe performances similar to the case where the model is tested on held-out observational regimes. This setting may be simple, but still clearly hints at a good robustness of the abstractions when tested on new interventional regimes.

\vfill

\subsection{Interpretability}

We consider the problem of classifying whether an integer $n\in\{0,\dots,999,999\}$ is divisible by $6$.
The input is the base-$10$ representation $(d_0,\dots,d_{K-1})\in\{0,\dots,9\}^K$ with $K=6$ digits (from ones to hundred-thousands), and the label is $y=\vone[6\mid n]\in\{0,1\}$.
Each digit $d_k$ is encoded as a $10$-dimensional one-hot vector, giving an input dimension of $K\times 10=60$.

We train a fully-connected MLP with architecture $[60,32,32,32,1]$: three hidden layers of $32$ neurons with ReLU activations, followed by a single output logit trained with binary cross-entropy loss.
Training uses the Adam optimizer (learning rate $10^{-3}$, batch size $128$) with early stopping (patience $25$, up to $200$ epochs) on an $80/20$ train/test split of all $10^6$ integers. The network reaches an almost perfect test accuracy.

We view the trained MLP as a low-level SCM $\mL$ whose endogenous variables $\mX$ are the post-ReLU activations of all hidden neurons, giving $n=|\mX|=32+32+32=96$ variables.
The causal graph $G_L$ of $\mL$ is determined by the network connectivity: every neuron in hidden layer $\ell$ is a parent of every neuron in hidden layer $\ell+1$ (and the last hidden layer feeds into the logit), forming a strictly layered DAG with layer sizes $[32,32,32]$. The mechanisms $\mF_L$ are the trained affine maps composed with ReLU (or identity for the logit), and the context $\mU_L$ corresponds to the input $n$ (encoded as one-hot).
Since the network is deterministic, the context (here, the input of the network) fully determines all activations.

To collect interventional data, we perform hard interventions on each of the $96$ hidden neurons individually (excluding the output logit).
For each neuron $X_j$ in hidden layer $\ell$, we clamp its post-ReLU activation to a value resampled from its empirical training distribution and record the resulting activations of all downstream neurons.
Concretely, for a subsample of $500$ inputs per intervention, we perform a forward pass where $X_j$ is replaced by a draw from its marginal: $X_j\leftarrow \tilde x_j$, with $\tilde x_j$ sampled uniformly from the training-set values of $X_j$. This yields $96 \times 500 = 48,000$ interventional samples, combined with $200,000$ observational samples (test-set forward passes) for a total of $248,000$ samples.

Among the $248,000$ samples, 80\% are used to train UCAD, and 20\% are used for validation and model selection.
We build a separate held-out observational set of $1,000$ samples for test, on which we report the metrics below.

When applying UCAD we search over $m\in\{2,3\}$ factors, and select the model with the lowest validation NLL among runs that achieved constraint feasibility (DAG, anchor, and sparsity constraints).
We evaluate the learned factors against \emph{positive} concepts: $c_2 = \vone[2\mid n]$, $c_3=\vone[3\mid n]$, $c_6=\vone[6\mid n]$, and $\tilde c_3 = \sum_{k} d_k$ (the sum of all digits, an intermediate quantity involved in the intuitive computation of divisibility by $3$); and \emph{negative} controls: $\vone[5\mid n]$, $\vone[7\mid n]$, $\vone[11\mid n]$.
Results are reported in Table \ref{tab:div6_results} below.

\begin{table}[htbp]
\centering
\caption{Aggregated test metrics assessing the reconstruction of positive and negative concepts by learned latent factors for the divisibility-by-6 task (mean $\pm$ std over 5 trained MLPs, of the UCAD model achieving the lowest feasible objective among 60 configs).}
\label{tab:div6_results}
\begin{tabular}{lc}
\toprule
Metric & 96 variables, 3 factors, [32, 32, 32]\\
\midrule \midrule
NLL & $-0.2955 \pm 0.3008$ \\
DAG Violation & $1.84\cdot 10^{-6} \pm 3.98\cdot 10^{-6}$ \\
\midrule
\midrule
\multicolumn{2}{c}{\textbf{Positive concepts}} \\
MCC (Pearson) & $0.4072 \pm 0.0904$ \\
MCC (Spearman) & $0.4212 \pm 0.0659$ \\
MCC (RDC) & $0.7210 \pm 0.0517$ \\
$R^2$ Linear & $0.2619 \pm 0.0703$ \\
$R^2$ MLP & $0.5948 \pm 0.0305$ \\
\midrule
\multicolumn{2}{l}{\textbf{Per-concept best Pearson correlation (positive concepts)}} \\
Divisibility by 2 & $0.8046 \pm 0.1805$ \\
Divisibility by 3 & $0.0581 \pm 0.0563$ \\
Divisibility by 6 & $0.3688 \pm 0.0892$ \\
Sum of all digits & $0.1098 \pm 0.0367$ \\
\midrule
\multicolumn{2}{l}{\textbf{Per-concept best RDC coefficient (positive concepts)}} \\
Divisibility by 2 & $0.8818 \pm 0.0768$ \\
Divisibility by 3 & $0.3161 \pm 0.1088$ \\
Divisibility by 6 & $0.5451 \pm 0.1323$ \\
Sum of all digits & $0.1575 \pm 0.0320$ \\
\midrule
\midrule
\multicolumn{2}{c}{\textbf{Negative controls}} \\
MCC (Pearson) & $0.0261 \pm 0.0093$ \\
MCC (Spearman) & $0.0284 \pm 0.0093$ \\
MCC (RDC) & $0.0654 \pm 0.0130$ \\
$R^2$ Linear & $0.0022 \pm 0.0016$ \\
$R^2$ MLP & $0.0269 \pm 0.0217$ \\
\midrule
\multicolumn{2}{l}{\textbf{Per-concept best Pearson correlation (negative controls)}} \\
Divisibility by 5 & $0.0614 \pm 0.0316$ \\
Divisibility by 7 & $0.0121 \pm 0.0064$ \\
Divisibility by 11 & $0.0131 \pm 0.0069$ \\
\midrule
\multicolumn{2}{l}{\textbf{Per-concept best RDC coefficient (negative controls)}} \\
Divisibility by 5 & $0.1416 \pm 0.0350$ \\
Divisibility by 7 & $0.0679 \pm 0.0330$ \\
Divisibility by 11 & $0.0292 \pm 0.0045$ \\
\bottomrule
\end{tabular}
\end{table}

We measure a stronger reconstruction of hypothesized concepts than with negative controls (average MCC-RDC of $0.72$ vs. $0.07$ over 5 experiments). In particular, one of the factors has an average Pearson correlation of $0.80$ with $c_2$, which is a strong signal that it is consistently implemented by the network. Reasonably high MCC-RDC coefficients can be observed for $c_6$ and $c_3$, especially compared to negative controls which are not reconstructed as well.
Interestingly, the negative control $c_5$ is reconstructed by one of the factors as much as the sum of all digits $\tilde c_3$ (average MCC-RDC of $0.14$ vs. $0.16$).

These results remain difficult to interpret, keeping in mind that our intuitive method to compute $c_3$ is recursive.
Furthermore, a low agreement with some concept may not necessarily indicate a failure of UCAD, but could simply mean that the DNN implements a different strategy to solve the problem, for example a heuristic that works only up to 1 million that might sometimes rely on divisibility by $5$ for some reason. This illustrates the difficulty of validating such a method: we look for tasks where we expect the DNN to implement a given logic, but our method might unveil that the network implements a different algorithm compared to what was expected. Thus, more work is required to interpret the factors that were learned.

\section{Additional acknowledgment}
The disks from Figure \ref{fig:HL_disks_schema} are freely adapted from Figure 1 of \citet{deleu2023joint}.

\end{document}